\documentclass[11pt]{article}
\usepackage{amsmath}
\usepackage{amssymb}
\usepackage{amsfonts}
\usepackage{fullpage}
\usepackage{url}
\usepackage{xspace}
\usepackage{fancybox}
\usepackage{amsthm}
\usepackage[utf8]{inputenc}
\usepackage{todonotes}

\newtheorem{theorem}{Theorem}
\newtheorem{lemma}[theorem]{Lemma}

\newtheorem{corollary}[theorem]{Corollary}
\newtheorem{definition}[theorem]{Definition}

\newtheorem{remark}[theorem]{Remark}
\newtheorem{fact}[theorem]{Fact}

\newcommand{\EquationName}[1]{\label{eq:#1}}

\newcommand{\LemmaName}[1]{\label{lem:#1}}
\newcommand{\DefinitionName}[1]{\label{def:#1}}
\newcommand{\CorollaryName}[1]{\label{cor:#1}}
\newcommand{\SectionName}[1]{\label{sec:#1}}
\newcommand{\TheoremName}[1]{\label{thm:#1}}
\newcommand{\RemarkName}[1]{\label{rem:#1}}
\newcommand{\FigureName}[1]{\label{fig:#1}}

\newcommand{\Equation}[1]{Eq.\:\eqref{eq:#1}}

\newcommand{\Lemma}[1]{Lemma~\ref{lem:#1}}

\newcommand{\Corollary}[1]{Corollary~\ref{cor:#1}}
\newcommand{\Section}[1]{Section~\ref{sec:#1}}
\newcommand{\Theorem}[1]{Theorem~\ref{thm:#1}}
\newcommand{\Remark}[1]{Remark~\ref{rem:#1}}
\newcommand{\Figure}[1]{Figure~\ref{fig:#1}}

\newcommand{\Eqsub}[1]{\eqref{eq:#1}}

\newenvironment{fminipage}
{
  \begin{Sbox}\begin{minipage}}
    {\end{minipage}\end{Sbox}\fbox{\TheSbox}
}

\newenvironment{algbox}[0]{
  \vskip 0.2in
  \noindent
  \begin{fminipage}{6.3in}
  }{
  \end{fminipage}
  \vskip 0.2in
}

\newcommand\erspud{\textbf{ER-SpUD}\xspace}
\newcommand\erspuddc{\textbf{ER-SpUD(DC)}\xspace}
\newcommand\erspuddctwo{\textbf{ER-SpUD(DCv2)}\xspace}
\newcommand\greedy{\textbf{Greedy}\xspace}

\DeclareMathOperator*{\E}{\mathbb{E}}
\let\Pr\relax
\DeclareMathOperator*{\Pr}{\mathbb{P}}
\let\P\relax
\DeclareMathOperator*{\P}{\mathbb{P}}
\DeclareMathOperator*{\supp}{support}
\DeclareMathOperator*{\polylog}{polylog}

\newcommand{\proofbelow}{3pt}
\newcommand{\afterproof}{\hfill $\blacksquare$ \par \vspace{\proofbelow}}
\newenvironment{proofof}[1]{\noindent\textbf{Proof} \,(of #1).\,}{\afterproof}

\newcommand{\eqdef}{\mathbin{\stackrel{\rm def}{=}}}
\newcommand\Oh{\mathcal{O}}
\newcommand\oh{o}
\newcommand\eps{\varepsilon}

\newcommand\R{\mathbb{R}}
\newcommand\N{\mathbb{N}}
\newcommand{\inprod}[1]{\left\langle #1 \right\rangle}
\newcommand\les\lesssim

\begin{document}

\author{Jaros{\l}aw B{\l}asiok\thanks{Harvard University, Cambridge, MA. \texttt{jblasiok@g.harvard.edu}. Supported by NSF grant IIS-1447471.}
  \and
Jelani Nelson\thanks{Harvard University, Cambridge, MA. \texttt{minilek@seas.harvard.edu}. Supported by NSF grant IIS-1447471 and CAREER award CCF-1350670, ONR grant N00014-14-1-0632 and Young Investigator award N00014-15-1-2388, and a Google Faculty Research Award.}
  }

\title{An improved analysis of the ER-SpUD dictionary learning algorithm}

\setcounter{page}{0}

\maketitle

\begin{abstract}
In {\em dictionary learning} we observe $Y = AX + E$ for some $Y\in\R^{n\times p}$, $A \in\R^{m\times n}$, and $X\in\R^{m\times p}$, where $p\ge \max\{n, m\}$, and typically $m \ge n$. The matrix $Y$ is observed, and $A, X, E$ are unknown. Here $E$ is a ``noise'' matrix of small norm, and $X$ is column-wise sparse. The matrix $A$ is referred to as a {\em dictionary}, and its columns as {\em atoms}. Then, given some small number $p$ of samples, i.e.\ columns of $Y$, the goal is to learn the dictionary $A$ up to small error, as well as the coefficient matrix $X$. In applications one could for example think of each column of $Y$ as a distinct image in a database. The motivation is that in many applications data is expected to sparse when represented by atoms in the ``right'' dictionary $A$ (e.g.\ images in the Haar wavelet basis), and the goal is to learn $A$ from the data to then use it for other applications.

Recently, the work of \cite{SpielmanWW12} proposed the dictionary learning algorithm \erspud with provable guarantees when $E = 0$ and $m = n$. That work showed that if $X$ has independent entries with an expected $\theta n$ non-zeroes per column for $1/n\lesssim \theta \lesssim 1/\sqrt{n}$, and with non-zero entries being subgaussian, then for $p\gtrsim n^2\log^2 n$ with high probability \erspud outputs matrices $A', X'$ which equal $A, X$ up to permuting and scaling columns (resp.\ rows) of $A$ (resp.\ $X$). They conjectured that $p\gtrsim n\log n$ suffices, which they showed was information theoretically necessary for {\em any} algorithm to succeed when $\theta \simeq 1/n$. Significant progress toward showing that $p\gtrsim n\log^4 n$ might suffice was later obtained in \cite{LuhV15}.

In this work, we show that for a slight variant of \erspud, $p\gtrsim n\log(n/\delta)$ samples suffice for successful recovery with probability $1-\delta$. We also show that without our slight variation made to \erspud, $p\gtrsim n^{1.99}$ samples are required even to learn $A, X$ with a small success probability of $1/\mathop{poly}(n)$. This resolves the main conjecture of \cite{SpielmanWW12}, and contradicts a result of \cite{LuhV15}, which claimed that $p\gtrsim n\log^4 n$ guarantees high probability of success for the original \erspud algorithm.
\end{abstract}

\section{Introduction}
The \emph{dictionary learning} or \emph{sparse coding} problem is defined as follows. There is a hidden set of vectors $a_1, a_2, \ldots a_m \in \mathbb{R}^n$ (called a ``dictionary''), with $\mathop{span}\{a_1, \ldots a_m\} = \mathbb{R}^n$. We are given a sequence of samples $y_i = A x_i + \epsilon_i$, where each $x_i$ is a sparse vector and $\epsilon_i$ is noise. In other words each $y_i$ is close to a linear combination of few vectors $a_k$. The goal is to recover both matrix $A$ and the sparse representations $x_i$. We can write it as a matrix equation
\begin{equation*}
Y = A X + E
\end{equation*}
where the vectors $y_i$ are the columns of $Y$, and $x_i$ are columns of $X$. Let $A \in \R^{n\times m}$ and $X \in \R^{m \times p}$. Traditionally, and as motivated by applications, the interesting regime of parameters is when $A$ is of full row rank (in particular $n \leq m$) \cite{AgarwalAJNT14}.

The dictionary learning problem is motivated by the intuition that the dictionary $A$ is in some sense the ``right'' spanning set for representing vectors $y_i$ since it allows sparse representation. In some domains this correct basis is known thanks to a deep understanding of the domain in question: for example the Fourier basis for audio processing, or Haar wavelets for images. Here we want to infer analogous ``nice'' representations of the data from the data itself.  As it turns out, even in situations such as audio and image processing in which traditional transforms are useful, replacing them with dictionaries learned directly from data turned out to improve quality of the solution (see for example \cite{EladA06}, which applied a dictionary learning algorithm for image denoising).

This problem has found a tremendous number of applications in various areas, such as image and video processing (e.g. \cite{MairalBPSZ09, BrytE08, EladA06}; see \cite{MairalBP14} for more references), image classification \cite{RainaBLPN07,MairalBPSZ08} as well as neurobiology \cite{LiYBXGS14}. Given its huge practical importance, a number of effective heuristics for dictionary learning were proposed \cite{AharonEB06, MairalBPS10} --- those are based on iterative methods for solving the (non-convex) optimization problem of minimizing the sparsity of $X'$ subject to $Y$ being close to $A'X'$. Some of these algorithms work well in practice but without provable guarantees.

\subsection{Prior work} 
\begin{figure}
	\begin{tabular}{|cccccc|}
                \hline
		ref & sample complexity & noise & overcomplete & sparsity & arbitrary dict. \\
		\hline
		\hline
		\cite{SpielmanWW12} & $\Oh(n^2 \log^2 n) $ & No & No & $\Oh(\sqrt{n})$ & Yes \\
                \hline
		\cite{AgarwalAJNT14} & $\Oh(m^2)$ & No & Yes & $ \Oh(n^{1/4})$ & No \\
                \hline
        \cite{AroraGM14} & $\Oh(m^2 s^{-2} + s^2 m)$ & Yes & Yes &  $ \Oh(\min(m^{2/5}, \frac{\sqrt{n}}{\log n}))$ & No \\ 
                \hline
		\cite{AroraGM14} & $\Oh(\mathrm{poly}(m))$ & Yes & Yes & $ \Oh(n^{1/2 - \epsilon})$ & No \\
                \hline
		\cite{AroraBGM14}* & $\Oh(\mathrm{poly}(m))$ & No & Yes & $\Oh(n/\mathrm{polylog}(n))$ & No \\
                \hline
		\cite{BarakKS15} & $\Oh(\mathrm{poly}(m))$ & Yes & Yes & $\Oh(n^{1-\epsilon})$ & Yes \\
                \hline
		\cite{BarakKS15}* & $\Oh(\mathrm{poly}(m))$ & Yes & Yes & $\Oh(n)$ & Yes \\
                \hline
        \cite{SunQW15} & $\Oh(\mathrm{poly}(m, \kappa(A))))$ & No & No & $\Oh(n)$ & Yes \\
        \hline
        \cite{VempalaX15} & $\Oh(\mathrm{poly}(n))$ & Yes & No & $\Oh(n)$ & Yes \\
                \hline
		\cite{LuhV15}\footnotemark & $\Oh(n \log^4 n)$ & No & No & $\Oh(\sqrt{n})$ & Yes \\
                \hline
		This work & $\Oh(n \log n)$ & No & No & $\Oh(\sqrt{n})$ & Yes \\
                \hline
	\end{tabular}
	\caption{
		Comparison of algorithms with proven guarantees for dictionary learning. Last column indicates whether the dictionary can be arbitrary, or if additional structure is assumed in order to guarantee recovery. Algorithms marked with star require quasi-polynomial running time. $\kappa(A)$ denotes condition number.}\FigureName{bigtable}
\end{figure}\footnotetext{As written, their work has certain errors which we discuss later in detail. Nevertheless, using some of our approaches we believe it should be possible to salvage their sample complexity bound in the Bernoulli-gaussian model for $X$, but not in the more general Bernoulli-subgaussian model (since in particular, $p\gtrsim n^{1.99}$ samples are required for that algorithm even to succeed with polynomially small {\em success} probability; see \Section{bernoulli-rademacher-lower-bound}.}
Until recently there was little theoretical understanding of the dictionary learning problem.  Spielman, Wang and Wright in \cite{SpielmanWW12} proposed the first algorithm that provably solves this problem in some regime of parameters. More concretely, they assumed no presence of noise (i.e. $E = 0$), and that $A$ is a basis (that is $n = m$), potentially adversarially chosen. The vectors $x_i$ are sampled independently at random from some distribution --- specifically, each entry $x_{i,j}$ is nonzero with probability $1-\theta$, and once it is nonzero, it is a symmetric subgaussian random variable (i.e.\ with tails decaying at least as fast as a gaussian), independent from every other entry. Henceforth we say that a matrix $X\in\R^{n\times p}$ follows the {\em Bernoulli-subgaussian model} with parameter $\theta$, if the entries  $X_{i,j}$ are i.i.d.\ with $X_{i,j} = \chi_{i,j} g_{i,j}$, where $\chi_{i,j} \in \{0,1\}$ are Bernoulli random variables with $\E \chi_{i,j} = \theta$, and $g_{i,j}$ are symmetric subgaussian random variables. We also say that $X$ follows the Bernoulli-Rademacher model if $g_{i,j}$ in the above definition are independent Rademachers (i.e.\ uniform $\pm 1$).

Under the Bernoulli-subgaussian model for $X$, \cite{SpielmanWW12} proved that once the number of samples $p$ is $\Omega(n \log n)$ and the sparsity $s = \theta n$ (i.e. expected number of nonzero entries in each column of $X$) is at least constant and at most $\Oh(n)$, the matrix $Y$ with high probability has a unique decomposition as a product $Y = AX$, up to permuting and rescaling rows of $X$. Moreover, the number of samples $p = \Omega(n \log n)$ was proven to be optimal in the constant sparsity regime $s = \Theta(1)$. In particular, it is possible in principle to find such a decomposition information-theoretically, but unfortunately not necessarily with an efficient algorithm.

In addition to the above, they proposed an efficient algorithm \erspud (\emph{Efficient Recovery of Sparsely Used Dictionaries}) to find this unique decomposition, in a more restricted regime of parameters. Namely, they proposed an algorithm and proved that it finds correctly the unique decomposition $Y = A X$, with high probability over $X$, as long as the sparsity $s$ is at least constant and at most $\Oh(\sqrt{n})$, and the number of samples $p$ is at least $\Omega(n^2 \log^2 n)$. The low sparsity constraint was inherent to their solution: according to the proof in the same paper, if $s=\Omega(\sqrt{n \log n})$ the algorithm with high probability fails to find the correct decomposition. They conjectured however, that with the number of samples $p$ as small as $\Oh(n \log n)$, \erspud should return the correct decomposition with high probability, matching the sample lower bound for when $s = O(1)$.

Since then, much more theoretical work has been dedicated to the dictionary learning problem; see \Figure{bigtable}. In the work of Agarwal et al.\ \cite{AgarwalAJNT14}, and independently Arora et al.\ \cite{AroraGM14}, an algorithm was proposed that works for overcomplete dictionaries $A$ (i.e. when $m > n$), under additional structural assumptions on $A$ --- namely that $A$ is \emph{incoherent}, i.e.\ the projection of any standard basis vector onto the column space of $A$ has small norm. The algorithm presented in \cite{AgarwalAJNT14} require $p=\tilde{\Oh}(m^2)$ samples, where $\tilde{\Oh}(f) = \Oh(f\cdot\log^{O(1)}(f))$. More detailed analysis of the dependence between sparsity and number of samples was provided in the work \cite{AroraGM14} for their algorithm --- for $s = \Oh(\min(\frac{\sqrt{n}}{\log n}, m^{2/5}))$, they require $\tilde{\Omega}(m^2 s^{-2} + m s^2)$ samples; if $s$ is larger than $m^{2/5}$, but smaller than $\Oh(\min(m^{1/2 - \varepsilon}, \frac{\sqrt{n}}{\log n}))$ the algorithm require $\Oh(m^C)$ samples, where $C$ is a large constant depending on $\varepsilon$.  In the lowest sparsity regime, i.e. $s = \Oh(\polylog(n))$, the sample complexity stated in their analysis simplifies to $\tilde{\Omega}(m^2)$, for comparison in the most favorable sparsity regime $s = \Theta(m^{1/4})$, the number of samples necessary for correct recovery is $\Omega(m^{3/2})$.  The work \cite{AroraGM14} also proves correct recovery by this algorithm in the presence of noise. Later Arora et al.\ \cite{AroraBGM14} gave a quasipolynomial time algorithm working for sparsity up to $\Oh(n/\mathrm{polylog}(n))$, but under much stronger assumptions on the structure of $A$. Those assumptions include in particular, that the dictionary $A$ itself is assumed to be sparse, which is violated in many natural examples, e.g.\ the discrete Fourier basis. They prove that their algorithm correctly recovers the hidden dictionary given access to $p = \Oh(m^C)$ samples, for some unspecified constant $C$.

Barak et al.\ \cite{BarakKS15} proposed an algorithm fitting in the Sum-of-Squares framework, which works in polynomial time for sparsity $\Oh(n^{1-\epsilon})$ and in quasipolynomial time for sparsity as large as $\Oh(n)$, again given access to $\Oh(m^C)$ samples for some unspecified constant $C$. Moreover, this algorithm works under the presence of noise and a more general model of $X$. In particular, coordinates within a single column are not required to be fully independent. Recently, Sun et al.\ \cite{SunQW15} proposed a polynomial time algorithm for the case when $n=m$ and sparsity is as large as $\Oh(n)$. Their result works in the similar model as in \cite{SpielmanWW12}, without any additional assumptions on the matrix $A$, and with matrix $X$ having independent entries that are product of Bernoulli and gaussian random variables (as opposed to the weaker subgaussian assumption in \cite{SpielmanWW12}). The sample complexity depends polynomially on $n$ and the condition number of the dictionary matrix $A$. In particular, in the low sparsity regime ($s = \Theta(\polylog(n))$), this sample complexity is as large as $\tilde{\Omega}(n^9)$ even if the matrix $A$ is well conditioned.

Work on Independent Component Analysis (ICA) \cite{FriezeJK96,NguyenR09,BelkinRV13,AroraGMS15,GoyalVX14,VempalaX15} is also relevant to the dictionary learning problem. In this problem, again one is given $Y  = AX+E$ for square $A$, with the assumption that the entries of $X$ are i.i.d.\ (and $X$ need not necessarily be sparse). The works in ICA then say that $A, X$ can be efficiently recovered using few samples, but where the sample complexity depends on the distribution of entries of $X$. For example in the case of Bernoulli-Rademacher entries with $\theta = 1/n$ (constant sparsity per column of $X$), these works require large polynomial sample complexity. For example, \cite[Theorem 1]{VempalaX15} implies a sufficient sample complexity in this setting of $p\gg n^{12}$.

From \Figure{bigtable}, one can see that the ``holy grail'' of dictionary learning is to achieve the following features simultaneously: (1) low sample complexity, i.e.\ nearly-linear in the dimension $n$ and number of atoms $m$, (2) the ability to handle noise (the more noise handled the better), (3) handling overcomplete dictionaries, (4) handling a larger range of sparsity, with $s = O(n)$ being the best, (5) making no assumptions on the dictionary $A$, (6) a fast algorithm to actually learn the dictionary from samples, and (7) making few assumptions on the matrix $X$.

Most of the aforementioned results focus on weakening the sparsity constraint under which it is possible to perform efficient learning, or handling overcomplete dictionaries or noise. These all, however, come at an expense: the number of samples necessary for those algorithms to provably work is quite large, often of order $n^C$ for some large constant $C$. Some of the algorithms also make strong assumptions on $A$, and/or have quasi-polynomial running time.

Recently, Luh and Vu in \cite{LuhV15} made significant progress toward showing that the \erspud algorithm proposed in \cite{SpielmanWW12} actually solves the dictionary learning problem already with $p = \Oh(n \log^4 n)$ samples. They claimed to prove that this $p$ in fact suffices for dictionary learning. In fact however, several probabilistic events were analyzed in \cite{SpielmanWW12}, and if they all occurred then \erspud performed correct recovery. The work \cite{LuhV15} analyzed arguably the most complex of these events more efficiently, showing a certain crucial inequality held with good probability when $p\gtrsim n\log^4 n$. Unfortunately there is a gap: \cite{SpielmanWW12} required this inequality to hold for exponentially many settings of variables, and thus one wants the inequality to hold for any fixed instantiation with very high probability to then union bound, and \cite{LuhV15} does not provide such a probabilistic analysis (see \Remark{exponential}). More seriously, there are other events defined in \cite{SpielmanWW12} which {\em require} $p\gtrsim n^2$ to hold whp in the Bernoulli-subgaussian model (except in the case the subgaussians are actual gaussians), and \cite{LuhV15} did not discuss these events at all (see for example \Remark{column-pair}). In fact, in \Section{bernoulli-rademacher-lower-bound} we prove that in the Bernoulli-Rademacher model the \erspud algorithm of \cite{SpielmanWW12} actually {\em requires} $p\gtrsim n^{1.99}$ to succeed with probability even polynomially small in $n$, contradicting the main result of \cite{LuhV15} which claimed $1-o(1)$ successful learning for $p$ nearly linear in $n$.

\medskip

\paragraph{Our contribution:} We very slightly modify the algorithm \erspud to obtain another polynomial-time dictionary learning algorithm ``\erspuddctwo'', which circumvents our $p\gtrsim n^{1.99}$ lower bound for \erspud in the Bernoulli-subgaussian model. We then show that \erspuddctwo provides correct dictionary learning with probability $1-\delta$ with sparsity $s = \Oh(\sqrt{n})$ as long as $p\gtrsim n\log(n/\delta)$. In particular our result shows that a slight modification of \erspud provides correct dictionary learning for complete dictionaries with no noise, which provably works with high probability using $p\gtrsim n\log n$ samples. This resolves the main open problem of \cite{SpielmanWW12}.

\medskip

Furthermore, the work of \cite{LuhV15} observed that the method of their proof is connected to generic chaining, but that after a certain point the methods ``become different in all aspects'' \cite[Section G]{LuhV15}. They also advertised and proved a new ``refined version of Bernstein’s concentration inequality for a sum of independent variables''. Unlike their work, our analysis has the benefit of using standard off-the-shelf concentration and chaining results, thus making the proof simpler and more easily accessible since it is less ad-hoc.

\subsection{Approach overview}

In \Figure{erspud} we give the algorithm \erspuddctwo analyzed in this work, a slight modification of \erspuddc from \cite{SpielmanWW12}. The only difference between DCv2 and the original DC variant in \cite{SpielmanWW12} is that we try all $\binom{p}{2}$ pairings of columns, whereas DC tried a random pairing of the $p$ columns into $p/2$ pairs. As we will see soon, one of the several conditions in \cite{SpielmanWW12} necessary for their proof of successful recovery of $(A, X)$ from $Y$ actually requires $p = \Omega(n^2)$ if using the DC variant (see \Remark{column-pair}), and hence our switch to DCv2 allows $p$ to be reduced to $\Oh(n\log n)$. In any case, this issue is easily circumvented by switching to DCv2 as we shall soon justify.

Henceforth when we refer to \erspud, we are referring to \erspuddctwo unless we state otherwise.

\begin{figure}
\begin{algbox}
\noindent \erspuddctwo:\texttt{ Exact Recovery of Sparsely-Used Dictionaries using the sum of two columns of $Y$ as constraint vectors.}
\begin{enumerate}
\item Create all $T = \binom{p}{2}$ pairings of columns of $\boldsymbol Y$ and for $j\in[T]$ write $g_j=\{\boldsymbol Y\boldsymbol e_{j_{1}},\boldsymbol Y\boldsymbol e_{j_{2}}\}$.
\item For $j=1\dots T$
\begin{enumerate}
\item [] Let $\boldsymbol r_j=\boldsymbol Y\boldsymbol e_{j_{1}}+\boldsymbol Y\boldsymbol e_{j_{2}}$, where $g_j = \left\{\boldsymbol Y\boldsymbol e_{j_{1}}, \boldsymbol Y \boldsymbol e_{j_{2}}\right\}  $.
\item [] Solve $\min_{\boldsymbol w} \; \| \boldsymbol w^T \boldsymbol Y\|_1  \text{ subject to }  \boldsymbol r_j^T \boldsymbol w = 1,$
   and set $\boldsymbol  s_{j} = \boldsymbol w^{T} \boldsymbol Y$. 
\end{enumerate}
\end{enumerate}
\end{algbox}
\vspace{-0.2in}
\begin{algbox}
\noindent \textbf{Greedy:}\texttt{ A Greedy Algorithm to Reconstruct $\boldsymbol X$ and $\boldsymbol A$.}
\begin{enumerate}
\item \textbf{REQUIRE:} $\mathcal S = \{ \boldsymbol s_1, \dots, \boldsymbol s_T \} \subset \R^p$. 
\item For $i=1\dots n$
\begin{enumerate}
\item [] REPEAT
\begin{enumerate}
\item[] $l \gets \arg\min_{\boldsymbol s_l \in \mathcal S}\| \boldsymbol s_l \|_0$, breaking ties arbitrarily
\item[] $\boldsymbol x_i=\boldsymbol s_l$
\item[] $\mathcal S=\mathcal S \backslash \{\boldsymbol s_l\}$
\end{enumerate}
\item[] \textbf{UNTIL} \texttt{rank([$\boldsymbol x_1,\dots, \boldsymbol x_i$])$=i$}
\end{enumerate}
\item Set $\boldsymbol X=[\boldsymbol x_1,\dots, \boldsymbol x_n]^T$, and $\boldsymbol A = \boldsymbol Y \boldsymbol Y^T(\boldsymbol X \boldsymbol Y^T)^{-1}$.
\end{enumerate}
\end{algbox}
\caption{\erspud recovery algorithm.}\FigureName{erspud}
\end{figure}

The main insight in the recovery analysis of \cite{SpielmanWW12} is that the last line of the \erspud pseudocode in \Figure{erspud} can be rewritten (only in the analysis, since $A,X$ are unknown) as $\min_w \|w^T AX\|_1$ subject to $(A(Xe_{j_1} + Xe_{j_2}))^T w = 1$. Then writing $z = A^T w$, this linear program (LP) is equivalent to the secondary LP $\min_z \|z^T X\|_1$ subject to $b_j^T z = 1$, since we could recover $w = (A^T)^{-1} z$ since $A$ is invertible. Here $b_j$ denotes $Xe_{j_1} + Xe_{j_2}$. The ideal case then is that the only optimal solution to the second LP will be a vector $z_*$ that is $1$-sparse. In this case, the solution to the LP that we {\em actually} solve is equal to $w_* = (A^T)^{-1} z_* = (z_*^T A^{-1})^T$ and thus a scaled row of $A^{-1}$, implying $w_*^T Y$ is a scaled row of $X$. Thus, if $z_*$ is $1$-sparse in the second LP, then the solution to the first LP allows us to recover a scaled row of $X$.

The work \cite{SpielmanWW12} then outlines certain conditions for $X$ that, if they hold, guarantee correct recovery of $(A, X)$. We now state these deterministic conditions, as per \cite{SpielmanWW12}, which imply correct recovery of $(A,X)$ via \erspud when they all simultaneously hold.

\begin{itemize}
\item[\bf(P0)] Every row of $X$ has positive support size at most $(10/9)\theta p$. Furthermore, every linear combination of rows of $X$ in which at least two of the coefficients in the linear combination are non-zero has support size at least $(11/9)\theta p$.
\item[\bf(P1)] For every $b$ satisfying $\|b\|_0 \le 1/(8\theta)$, any solution $z_*$ to the optimization problem
\begin{equation}
\min \|z^T X\|_1 \text{ subject to } b^T z = 1 \EquationName{secondary-lp}
\end{equation}
has $\supp(z_*) \subseteq \supp(b)$.
\item[\bf(P2)] Let $q$ be $\frac{1}{8\theta}$. For every $J\in\binom{[n]}q$ and every $b\in\R^n$ satisfying $|b|_{(2)}/|b|_{(1)} \le 1/2$, the solution to the restricted problem
\begin{equation}
\|z^T X_{J,*}\|_1 \text{ subject to } b^T z = 1 \EquationName{secondary-lp'}
\end{equation}
is unique, $1$-sparse, and is supported on the index of the largest entry of $b$. Here $|b|$ is the vector whose $i$th entry is $|b_i|$, and $|b|_{(j)}$ is the $j$th largest entry of $|b|$. Also, $X_{J,*}$ denotes the submatrix of $X$ with rows in $J$.
\item[\bf(P3)] For every $i\in[n]$ there exist a pair of columns $Xe_{j_1}$ and $Xe_{j_2}$ in $X$ such that for $b = Xe_{j_1} + Xe_{j_2}$ with support $J$, we have that $0<|J| \le 1/(8\theta)$, $|b|_{(2)}/|b|_{(1)} \le 1/2$, and the unique largest entry of $|b|$ has index $i$.
\end{itemize}

The main result of \cite{SpielmanWW12} is then obtained by proving the following theorem, and then by showing that {\bf(P0)}--{\bf(P3)} all hold whp for $p\gtrsim n^2\log^2 n$.

\begin{theorem}[{\cite{SpielmanWW12}}]\TheoremName{correctness}
Suppose conditions {\bf(P0)}--{\bf(P3)} all hold. Then \erspud and \textbf{Greedy} from \Figure{erspud} recover $(A', X')$ such that $X' = \Pi D X$ and $A = A D^{-1} \Pi^{-1}$ for some diagonal scaling matrix $D$ and permutation matrix $\Pi$. That is, the recovered $(A', X')$ are correct up to scaling and permuting rows (resp.\ columns) of $X$ (resp.\ $A$).
\end{theorem}

It was implicit in \cite{SpielmanWW12}, and made explicit in \cite{LuhV15}, that to analyze the probability {\bf (P1)} holding as a function of $p$, it suffices to prove some upper bound on some stochastic process. Namely, \cite{LuhV15} proves that for $\Pi$ a Bernoulli-subgaussian matrix with $p$ rows, for $p = \Omega(n \log^4 n)$
\begin{equation}
	\Pr\left(\sup_{\|v\|_1=1} |\|\Pi v\|_1 - \E \|\Pi v\|_1| < c_0 \mu_{min}\right) > 1 - \oh(1)
	\EquationName{stoch-process}
\end{equation}
for some constant $c_0 < 1$, and $\mu_{min} := \inf_{\|v\|_1 = 1} \E \|X^T v\|_1$. Both \cite{SpielmanWW12,LuhV15} though required the stochastic process of \Equation{stoch-process} to be bounded for roughly $\binom{n}{1/(8\theta)}$ choices of $\Pi$, formed by taking various submatrices of $X^T$. The naive approach is to then argue that the inequality holds with failure probability $\ll 1/\binom{n}{1/(8\theta)}$ for a fixed $\Pi$ so then union bound over all such submatrices. Unfortunately the failure probability in \cite{LuhV15} was not made explicit and was only given as $\oh(1)$. In fact, it is likely that making the failure probability explicit would force $p\gg n^{3/2}$ for some sparsity settings (see \Remark{exponential}).

We show that, first of all, {\bf (P1)} can be relaxed to some {\bf (P1')} such that it suffices to only show \Equation{stoch-process} holds for polynomially many submatrices of $X$; showing {\bf (P1')} suffices requires only a very minor change in the previous analysis of \cite{SpielmanWW12}. Next, more importantly, show that $p\gtrsim n\log(n/\delta)$ suffices for \Equation{stoch-process} to hold with probability $1-\delta$. This is one of our main technical contributions, and is established using a generic chaining argument \cite{Talagrand14}. It is worth pointing out that simpler chaining inequalities, such as Dudley's inequality, would yield suboptimal results in our setting by logarithmic factors.

Next, we also show that {\bf (P2)} can be weakened to some other event {\bf (P2')} that holds whp as long as $p\gtrsim \theta^{-1}\log(n/\delta)$. Establishing this only requires a minor change in the analysis of \cite{SpielmanWW12}. 

Finally, in \Lemma{p3} we show that event {\bf (P3)} holds whp for $p \gtrsim n\log(n/\delta)$. This is the part where the modification of the algorithm was necessary, so that pairs of columns $X e_{j_1}$ and $X e_{j_2}$ mentioned in this condition refers to all $\binom{p}{2}$ pairs of columns, as opposed to a fixed pairing (with $\lfloor \frac{p}{2} \rfloor$ pairs). Note that this condition actually fails to hold for the unmodified version of the algorithm with $p \ll n^2$, for example when the matrix $X$ is drawn from the Bernoulli-Rademacher model, which is the main reason the unmodified algorithm fails to perform recovery (see \Section{bernoulli-rademacher-lower-bound}).

\subsection{Recent and independent work}
In a recent and independent work, Adamczak showed a main result similar to ours \cite{Adamczak16}. In particular, he showed that by making the same modification to \erspud that we have made (\erspuddctwo), $p\gtrsim n\log n$ suffices for successful dictionary learning with probability $1 - 1/(n\log n)$. Unlike our analysis which is based on Bernstein's inequality and generic chaining, the proof in \cite{Adamczak16} combines Bernstein's inequality with Talagrand's contraction principle, which leads to an overall simpler proof than ours. The main differences in the results themselves are that attention in \cite{Adamczak16} was not given to dependence of $p$ on the failure probability $\delta$, and the analysis of our \Section{bernoulli-rademacher-lower-bound} that \erspuddc fails for $p\ll n^2$ also does not appear there, so that our stated results are slightly stronger in these regards.

\section{Sufficient conditions for successful recovery \SectionName{conditions}}

We first explain why all conditions {\bf(P0)}--{\bf(P3)} holding simultaneously implies \erspud correctly recovers $(A, X)$. This argument appears in \cite{SpielmanWW12}, but since it is quite short we repeat it here for the benefit of the reader. Afterward, we slightly change {\bf(P1)} and {\bf(P2)} to similar conditions {\bf(P1')}, {\bf(P2')} which still suffice for correct operation of \erspud, and we show that all conditions above (with {\bf(P1)} replaced by {\bf(P1')} and \textbf{(P2)} by \textbf{(P2')}) hold simultaneously with probability $1-\delta$ as long as $p\gtrsim n\log(n/\delta)$. For all the conditions except {\bf(P0)}, the original analysis of \cite{SpielmanWW12} required $p\gg n^2$, which we cannot afford here, and hence we provide more efficient analyses here.

\medskip

\begin{proofof}{\Theorem{correctness}}
We first show that for every row $X_{i,*}$ of $X$, there is some $j$ so that $s_j$ from the output of \erspud is some scaling of $X_{i,*}$. By \textbf{(P3)} there is some pair of columns $X e_{j_1}, X e_{j_2}$ so that their sum $b$ has support $J$ with $0<|J|\le 1/(8\theta)$, and $|b|_{(2)}/|b|_{(1)} \le 1-\gamma_0$, and furthermore the unique largest entry of $|b|$ is at index $i$. Since $|J| \le 1/(8\theta)$, \textbf{(P1)} implies any solution $z_*$ to \Eqsub{secondary-lp} has support contained in $J$. Therefore \Eqsub{secondary-lp} has the same set of optimal solutions as \Eqsub{secondary-lp'}. By \textbf{(P2)} we thus know that the optimal solution is some $z_*$ which is $1$-sparse, supported only on index $i$. Therefore the corresponding $w_*$ obtained from \erspud is some scaling of $X_{i,*}$.

The above only shows all rows of $X$ appear as some $s_j$ (possibly scaled). However, many $s_j$ found may not be any scaled row of $X$ at all. We now complete the proof. First, observe \textbf{(P0)} implies $X$ has rank $n$ (if not, then either some row of $X$ is zero, which \textbf{(P0)} forbids, or some linear combination of at least two rows is zero, but the zero vector has sparsity $0<(11/9)\theta p$, and thus this also cannot happen). Therefore, the $n$ rows of $X$ are exactly the $n$ sparsest vectors in the rowspace of $X$ (up to scaling). Since they all appear as outputs of \erspud, scaled, they are then exactly the $n$ rows returned by \greedy in some order. Thus \greedy returns $X' = \Pi D X$ as desired. Noting $Y = AX$, we see $A = Y X^T (XX^T)^{-1}$. Meanwhile, \greedy returns
$$ A' = YX'^T (X' X'^T)^{-1} = Y X^T D\Pi^T (\Pi^T)^{-1} D^{-1} (XX^T)^{-1} D^{-1} \Pi^{-1} = A D^{-1} \Pi^{-1}. $$
\end{proofof}

\begin{remark}
\textup{
It is worth noting that the proof of \Theorem{correctness} implies that \greedy could be replaced by the following simpler algorithm and still maintain correctness under {\bf(P0)}--{\bf(P3)}: for each $s_j$ in order, remove any other $s_{j'}$ which are scaled copies of $s_j$, then return the $n$ sparsest $s_j$ remaining to be the rows of $X$.
}
\end{remark}

In the proof of \Theorem{correctness}, observe that \textbf{(P1)} is not invoked for {\it every} one of the possible sparsity patterns for $b$ (of which there are at least $\binom{n}{q}$ where $s = 1/(8\theta)$), and \textbf{(P2)} is not invoked for all possible choices of $J$. Rather, in the proof, the effects of \textbf{(P1)} and \textbf{(P2)} are only needed for the at most $\binom{p}{2}$ vectors $b$ that are non-zero, at most $1/(8\theta)$-sparse, and expressible as the sum of two columns of $X$. We now define \textbf{(P1')}, \textbf{(P2')} as follows.
\begin{itemize}
\item[\bf(P1')] For every $b$ that can be expressed as the sum of two columns of $X$,
\begin{equation}
\forall v\in\R^{|\bar{J}|},\ \|v^T X_{\bar{J},*}\|_1 - 2\|v^T X_{\bar{J},S}\|_1 > C p\sqrt{\frac{\theta}{|\bar{J}|}}\|v\|_1
\end{equation}
and
\begin{equation}
|S| < p/4
\end{equation}
where $C> 0$ is some fixed constant, $J = \supp(b)$, $\bar{J} = [n]\backslash J$, and $S\subseteq [p]$ is the set of columns of $X$ with support intersecting $J$.
\item[\bf(P2')] Let $q$ be $\frac{1}{8\theta}$. For every $b$ equaling the sum of two columns of $X$ and with $J \subset[n]$ its support, let $b'\in\R^{|J|}$ be the projection of $b$ onto its support. If $0< |J| \le q = 1/(8\theta)$ and $|b|_{(2)}/|b|_{(1)} \le 1/2$, then the solution to the restricted problem
\begin{equation}
\|z^T X_{J,*}\|_1 \text{ subject to } (b')^T z = 1 
\end{equation}
is unique, $1$-sparse, and is supported on the index of the largest entry of $b'$. Here $|b'|$ is the vector whose $i$th entry is $|b'_i|$, and $|b'|_{(j)}$ is the $j$th largest entry of $|b'|$. Also, $X_{J,*}$ denotes the submatrix of $X$ with rows in $J$.
\end{itemize}
The following corollary then is immediate from the proof of \Theorem{correctness} and the fact that \textbf{(P1')} implies that, for any $b\neq 0$ with $|\bar{J}| = \Omega(n)$ (which holds for $0<|J| < 1/(8\theta) = O(\sqrt{n})$ as per \textbf{(P3)}), it holds that the optimal solution $z_*$ to $\min \|z^T X\|_1$ subject to $b^T z = 1$ has $\supp(z_*) \subseteq \supp(b)$ (see the proofs of \cite[Lemma 11]{SpielmanWW12} and \cite[Lemma V.2]{LuhV15}).

\begin{corollary}\CorollaryName{correctness}
Suppose conditions {\bf(P0)}, {\bf (P1')}, {\bf (P2')}, and {\bf (P3)} all hold. Then \erspud and \textbf{Greedy} from \Figure{erspud} recover $(A', X')$ such that $X' = \Pi D X$ and $A = A D^{-1} \Pi^{-1}$ for some diagonal scaling matrix $D$ and permutation matrix $\Pi$. That is, the recovered $(A', X')$ are correct up to scaling and permuting rows (resp.\ columns) of $X$ (resp.\ $A$).
\end{corollary}

We now show {\bf(P0)}, \textbf{(P1')}, \textbf{(P2')}, and \textbf{(P3)} all simultaneously hold with probability $1-\delta$ as long as $p\gtrsim n\log(n/\delta)$ and $1/n \lesssim \theta \lesssim 1/\sqrt{n}$, which when combined with \Corollary{correctness} implies that \erspud has the desired correctness guarantee under this same regime for $p, \theta$.

\begin{theorem}\TheoremName{prob-analysis}
For $p\gtrsim n\log(n/\delta)$ and $1/n\lesssim \theta \lesssim 1/\sqrt{n}$,
\begin{equation}
\Pr(\neg{\bf(P0)}\vee \neg{\bf(P1')}\vee \neg{\bf(P2')}\vee \neg{\bf(P3)}) < \delta \EquationName{failure-prob}
\end{equation}
\end{theorem}
\begin{proof}
We will show the right hand side of \Eqsub{failure-prob} is at most $C\delta$ for some $C$, then the theorem follows by rescaling $\delta$. We use the union bound.

First, $\Pr(\neg{\bf(P0)}) < \delta$ was already shown, even under the weaker conditions $p\gtrsim n\log n + \theta^{-1}\log(n/\delta)$ and $1/n\lesssim \theta \le 1$, in \cite[Theorem 3]{SpielmanWW12}. We thus do not provide an analysis here.

For \textbf{(P1)}--\textbf{(P3)}, the analyses in \cite{SpielmanWW12} required $p \gg n^2$ for any non-trivially small failure probability. We thus now provide our analyses for \textbf{(P1')}, \textbf{(P2')}, and \textbf{(P3)}. Relaxing the requirement on $p$ for \textbf{(P3)} to hold with high probability required us to switch from \erspuddc to \erspuddctwo.

For \textbf{(P1')}, the analysis is almost identical to the proofs of \cite[Lemma 11]{SpielmanWW12} and \cite[Lemma V.2]{LuhV15} regarding \textbf{(P1)}. We repeat the slightly modified argument here for \textbf{(P1')}. Let $b$ be a particular sum of two columns of $X$. We will show that the condition of \textbf{(P1')} fails to hold for $b$ with probability at most $\delta/p^2$, which implies $\Pr(\neg\textbf{(P1')}) \le \delta$ by a union bound over all $\binom{p}{2}$ such $b$. Let $J, S$ be as in the definition of \textbf{(P1')} above. Define the event $\mathcal{E}_S$ as the event that $|S| < p/4$. Since $\theta n \le c\sqrt{n}$ for some small $c>0$, if $b = X_{*,j_1} + X_{*, j_2}$, it follows that any column index $j\notin\{j_1, j_2\}$ has support intersecting $J$ with probability at most $1/10$ (by making $c$ sufficiently small). Thus $\E|S| < p/10$, implying $\Pr(\neg \mathcal{E}_S) =  \Pr(|S| \ge p/4)$ is at most $\exp(-\Omega(p)) \le \delta/p^2$ by the Chernoff bound and fact that $p \gtrsim \log(p^2/\delta)$.

The definition of $\mathcal{E}_N$ is the following event:
\begin{equation}
\forall v\in\R^{|\bar{J}|},\ \|v^T X_{\bar{J},*}\|_1 - 2\|v^T X_{\bar{J},S}\|_1 > C p\sqrt{\frac{\theta}{|\bar{J}|}}\|v\|_1  \EquationName{event-en}
\end{equation}
for some constant $C$, where $\bar{J}$ denotes $[n]\backslash J$. Note though that $X_{\bar{J},*}$ is itself a matrix of i.i.d.\ Bernoulli-subgaussian entries (except for the two columns $j_1,j_2$, which are both zero). Thus setting $\Pi = X_{\bar{J},*}^T$ and applying \Theorem{main} with our choice of $p$, with probability at least $1 - \delta/p^2$, for all $v\in B_1$,
\begin{equation}
\|v^T X_{\bar{J},*}\|_1 \ge \frac 78 \E \|v^T X_{\bar{J},*}\|_1 = \frac{7p}8\E |v^T (X_{\bar{J},*})_{*,1}| \eqdef \frac {7p}8 \alpha(v) , \EquationName{en1}
\end{equation}
where $(X_{\bar{J},*})_{*,1}$ clumsily denotes the first column of the matrix $X_{\bar{J},*}$. The last inequality follows from \cite[Lemma 16]{SpielmanWW12}. Also, conditioned on $\mathcal{E}_S$, $|S| < p/4$. Let $X'$ be the matrix $X_{\bar{J}, S}$ padded with $p/4 - |S|$ additional columns, each independent of but identically distributed to the columns of $X$. Then, even conditioned on $\mathcal{E}_S$, $X'$ is a $|\bar{J}| \times p/4$ matrix of i.i.d.\ Bernoulli-gaussian entries (except for two columns which are both identically zero, corresponding to $j_1, j_2$). Thus applying \Theorem{main} to $\Pi = (X')^T$, with probability at least $1 - \delta/p^2$, for all $v\in B_1$,
\begin{equation}
\|v^T X'\|_1 \le \frac 32 \E \|v^T X'\|_1 = \frac{3p}8\E |v^T X'_{*,1}| = \frac {3p}8 \alpha(v) . \EquationName{en2}
\end{equation}

Then by combining \Eqsub{en1}, \Eqsub{en2} and scaling by $\|v\|_1$, we see that the left hand side of \Eqsub{event-en} is at least
\begin{equation}
\frac p8 \alpha(v) \gtrsim p \sqrt{\frac{\theta}{|\bar{J}|}}\|v\|_1 ,\EquationName{almost-there}
\end{equation}
with the inequality following from \cite[Lemma 16]{SpielmanWW12}.

We finally now analyze the probability that \textbf{(P2')} holds. It is implied by the proof of \cite[Lemma 12]{SpielmanWW12} that for \textbf{(P2')} to hold, it suffices for the following three equations to hold, where $\mathcal{B} = \{Xe_{j_1} + Xe_{j_2} : 1 \le j_1 < j_2 \le p \}$ is the set of all sums of pairs of columns in $X$:
\begin{align}
&\ \|X\|_{\ell_\infty\rightarrow\ell_\infty} \le (1+\eps) \mu\theta p \EquationName{condp21}\\
\forall j\in[n], &\ \|X_{[n]\backslash\{j\}, \bar{T}_j}\|_{\ell_\infty\rightarrow\ell_\infty} \le \alpha\mu \theta p \EquationName{condp22}\\
\forall J\subseteq [n] \text{ s.t. } \exists b\in\mathcal{B},\ J = \supp(b),&\ \|X_{j, \Omega_{J,j}}\|_1 \ge \beta \mu\theta p .\EquationName{condp23}
\end{align}
for some particular constants $\eps, \alpha, \beta > 0$ (specifically \cite{SpielmanWW12} pick $\beta = 7/8$, $\alpha = \eps = 1/8$). Also, for $X_{i,j} = \chi_{i,j} g_{i,j}$, $\mu$ denotes the constant $\E|g_{i,j}|$. Here $\|M\|_{\ell_p\rightarrow\ell_p}$ is the $\ell_p$ to $\ell_p$ operator norm, which in the case of $p = \infty$ is simply the largest $\ell_1$ norm of any row of $M$. Furthermore, $T_j = \{i : X_{j,i} = 0\}$, and
$$
\Omega_{J, j} = \{\ell : X_{j,\ell}\neq 0\text{ and } X_{j', \ell} = 0,\ \forall\ j'\in J\backslash\{j\}\} .
$$
It is already shown in \cite[Lemma 18]{SpielmanWW12} that \Eqsub{condp21} and \Eqsub{condp22} fail to hold with probability at most $Cn\exp(-c \theta p) < \delta$, by choice of $p$, where the constant $c$ depends on $\alpha, \eps$. For condition \Eqsub{condp23}, the proof of \cite[Lemma 12]{SpielmanWW12} shows that for any fixed $J$ as in \Eqsub{condp23} and $j\in J$ (see Eqn.~(56) of \cite{SpielmanWW12}),
$$
\Pr(\|X_{j, \Omega_{J,j}}\|_1 \le \beta \mu\theta p) \le 4 \exp\left(-\frac{c\theta p}{256}\right)
$$
Thus by a union bound over all $b\in \mathcal{B}$, \Eqsub{condp23} holds with probability at least $1-\delta$.

We upper bound $\Pr(\neg\textbf{(P3)})$ separately in \Lemma{p3}.
\end{proof}

\begin{remark}\RemarkName{exponential}\textup{
The work \cite{LuhV15} showed a weaker version of \Theorem{main} in which $p$ was required to be $\Omega(n\log^4 n)$, and where the failure probability was shown to be some non-explicit value $\delta = o(1)$. \cite{LuhV15} then claimed that this was sufficient to show that \erspuddc was correct with probability $1-o(1)$. Unfortunately, it appears there were a few gaps in their analysis. First, \cite{LuhV15} relied on conditions \textbf{(P2)} and \textbf{(P3)} from \cite{SpielmanWW12} (as rewritten above) both holding, but the only known probabilistic analyses of these conditions, given in \cite{SpielmanWW12}, required $p\gg n^2$. Secondly, the proof sketch of \cite[Lemma V.2]{LuhV15} showing that the condition of \Eqsub{random-process} suffices to imply \textbf{(P1)} actually invoked the inequality
$$
\sup_{v\in B_1} \left| \| \Pi v \|_1 - \E \|\Pi v\|_1 \right| \le \varepsilon \cdot \E \| \Pi v \|_1
$$
for at least $\binom{n}{q}$ choices of $\Pi$ where $s = 1/(8\theta)$ (specifically $\Pi = X_{*, S}$ for $\binom{n}{q}$ choices of $S$). Thus to apply a probabilistic inequality of the form \Eqsub{random-process} to imply \textbf{(P1)}, one actually needs a specific $\delta$ and not just $\delta=o(1)$; in particular one needs $\delta \ll 1/\binom{n}{q}$ to union bound over all $S$, which for the largest value of $\theta \simeq 1/\sqrt{n}$, means one needs $\delta \ll \exp(-C \sqrt{n}\log n)$. Thus even if \Eqsub{random-process} held for $p \gtrsim n\log(1/\delta)$, one would still need $p\gtrsim n^{3/2}\log n$ for the analysis there to imply that \textbf{(P1)} holds with positive probability.
}
\end{remark}

We are now going to show that with probability $1-\delta$ condition {\bf (P3)} holds. Let us present an intuition behind the proof before we delve into technical details. 

Consider the special case that $X_{ij} = b_{ij} g_{ij}$ with Bernoulli random variable $b_{ij}$ and independent \emph{continuous} subgaussian random variable $g_{ij}$. In such a case there would exist some fixed threshold $t_0$, such that $\P(|X_{ij}| > t_0) = \frac{1}{n}$ --- it would mean that a constant fraction of columns would have \emph{unique} entry larger than this threshold. For a single index $i \in [n]$ we would expect that at least $C \frac{p}{n} > \log \frac{n}{\delta}$ columns have a unique entry larger than $t_0$ and such that this entry has index $i$. Let us focus on this set of columns. If supports of any two such columns had common intersection exactly equal to $\{i\}$ --- and if the sign on this $i$-th coordinate were matching, then in fact sum of those two columns would exhibit a factor two gap between the largest and the second largest entry, with largest entry being on the $i$-th position --- indeed, entry on position $i$ would have magnitude larger than $2t_0$, whereas all other entries are at most $t_0$ in absolute value. We can expect to find such a pair with probability $1-\frac{\delta}{n}$, as all columns are expected to be $\Oh(\sqrt{n})$ sparse --- therefore for a fixed pair containing $\{i\}$, their supports would intersect on exactly $\{i\}$ with constant probability. We then prove that there exist such a pair with probability at least $\frac{\delta}{n}$ for every fixed $i$, and hence by union bound property {\bf(P3)} holds with probability $\delta$.

In the actual proof we do not assume that $g_{ij}$ is continuous, and hence a threshold $t_0$ for which $\P(|X_{ij}| > t_0) = \frac{1}{n}$ might not exist, and the proof is slightly more complicated, but it follows the same general intuition.

\begin{lemma}\LemmaName{p3}
    Let $X \in \R^{n\times p}$ be a Bernoulli-Subgaussian matrix with $\theta = \Oh(\frac{1}{\sqrt{n}})$. If $p = \Omega(n \log \frac{n}{\delta})$, then with probability at least $1-\delta$ condition {\bf(P3)} holds.
\end{lemma}
\begin{proof}
	Take $t_0 := \inf \{ t : \P(|X_{ij}| > t) < \frac{1}{n} \}$. Observe that  $\P(|X_{ij}| > t_0) \leq \frac{1}{n} \leq \P(|X_{ij}| \geq t_0)$.

	Let us define $y_k$ to be the $k$-th column of $X$. Let $l_k$ be the number of coordinates of $y_k$ strictly larger than $t_0$, and $s_k$ be the number of coordinates of $y_k$ of size larger or equal to $t_0$. We will say that column $k$ is well-separated if $l_k \leq 1 \leq s_k$.

	We claim that for fixed $k$, the probability that $y_k$ is well-separated is bounded away from zero.  Clearly probability that a column is not well-separated is equal to $\P(l_k > 1) + \P(s_k = 0)$. Let us consider two cases. If $\P(|X_{ij}| > t_0)  < \frac{1}{5n}$, then $\E l_k \leq \frac{1}{5}$ and $\P(l_k > 1) \leq \frac{1}{5}$ by Markov inequality. Moreover
	\begin{equation*}
		\P(s_k = 0) = 1 - \prod_j (1 - \P(X_{j,k} \geq t_0)) \leq 1 - \left(1 - \frac{1}{n}\right)^n \leq 1 - e^{-1}\left(1-\frac{1}{n}\right)
	\end{equation*}

	And finally probability that a single column is not well-separated is at most $1 - e^{-1}(1-\frac{1}{n}) + \frac{1}{5}$, which is at most $0.9$ for sufficiently large $n$.

	On the other hand, if $\P(|X_{ij}| > t_0) \geq \frac{1}{5n}$, we will prove that with probability bounded away from zero, there is exactly one entry of a column $y$ that is greater than $t_0$. 

	Indeed, if $\alpha := \P(|X_{ij}| > t_0)n$, we have
	\begin{align*}
		\P(l_k = 1) & = n\frac{\alpha}{n}(1-\frac{\alpha}{n})^{n-1} \\
		& = (1 - \frac{\alpha}{n})^{-1}\alpha(1 - \frac{\alpha}{n})^{n} \\
		& \geq \alpha \left(1-\frac{\alpha}{n}\right)^{-1} \left(1 - \frac{\alpha^2}{n}\right) e^{-\alpha}
	\end{align*}

	We know that $\alpha \in [\frac{1}{5}, 1]$, and therefore the last expression is bounded away from zero, for large enough $n$.

    Let $C \in (0, 1)$ be a constant, such that the probability for a column to be well-separated is at least $C$. We will prove now that with $\theta = \Oh(1/\sqrt{n})$ (where constant hidden in $\Oh$ notation depends on $C$), the probability that a column has support of size larger than $\frac{\sqrt{n}}{4}$ is at most $\frac{C}{2}$.

	Indeed, let $S_k \subset[n]$ be the support of the $k$-th column of $X$. We can assume that $\theta$ is such that $n \theta <  \frac{C\sqrt{n}}{8}$, so that $\E |S_k| = n\theta < \frac{C\sqrt{n}}{8}$. Now by Markov inequality $\P(|S_k| > \frac{\sqrt{n}}{4}) = \P(|S_k| > \frac{\E |S_k|}{C/2})< C/2$.

	Now, by union bound, any fixed column $k$ of matrix $X$ simultaneously is well-separated and have support of size at most $\frac{\sqrt{n}}{4}$ with probability at least $\frac{C}{2}$. Let us define a set $W \subset [p]$, such that $k \in W$ if and only if column $X_{*,k}$ is well-separated and $|\supp(X_{*,k})| < \frac{\sqrt{n}}{4}$. 

	Fix some index $i \in [n]$. We wish to prove that with probability at least $1 - \frac{1}{n\delta}$ there exist a pair of columns $j, k$ such that for $b = X_{*,j} + X_{*,k}$, we have $b_i > 2b_l$ for all $l\not=i$, and moreover $|\supp(b)| < \frac{\sqrt{n}}{2}$. Indeed, let $W_i \subset W$, be a set of those indices $k\in W$, such that $X_{i, k} \geq t_0$. Observe that for a fixed $k \in [p]$, we have $\P(k \in W_i) = \P(k \in W) \P(k \in W_i | k \in W) \geq \frac{C}{2} \cdot \frac{1}{n}$. That is $\E |W_i| \geq \frac{Cp}{2n}$, and by Chernoff bound $\P(W_i < \E |W_i|/2) < \exp(-\Oh(\frac{p}{n})) = \frac{\delta}{2n}$ for $p =\Omega(n \log \frac{n}{\delta})$.

    Let us condition on the event that $|W_i| > \frac{Cp}{4n}$. We take $W^+_i = \{ k \in W_i : X_{i,k} > 0\}$, and let us assume without loss of generality that $|W^+_i| \geq \frac{1}{2} |W_i|$ (otherwise we can use a symmetric argument on the complement --- that is, on the set of columns for which $X_{i,k}$ is negative). We wish to prove that with probability at least $1 - \frac{\delta}{2n}$ there exist a pair of indices $j_1, j_2 \in W^{+}_i$, such that $S_{j_1} \cap S_{j_2} = \{i\}$.

    Observe that for $k \in [p]$, $S_k \setminus \{i\}$ conditioned on the fact that $k \in W^{+}_i$ and $|W_i| \geq \frac{Cp}{4n}$ has distribution supported on sets of size at most $\frac{\sqrt{n}}{4}$ and permutationally invariant. Therefore a pair of two independent such sets is disjoint with probability bounded at least $1 - n \left(\frac{\sqrt{n}}{4}\right)^2 = \frac{15}{16}$. Consider an arbitrary pairing of elements in $W^{+}_i$. We have $\lfloor \frac{|W_i|}{2}\rfloor$ pairs of indices $(k_1^j, k_2^j)$, which is at least $\frac{Cp}{8n} = \Omega(\log \frac{n}{\delta})$, and therefore with probability at least $1 - \frac{\delta}{2n}$ there is a pair such that $S_{k_1^j} \cap S_{k_2^j} = \{i\}$. Take $b := X_{*, k_1^j} +X_{*, k_2^j}$; we have $b_i \geq 2 t_0$ because both $k_1^j, k_2^j \in W_i^+$. On the other hand for $i'\not=i$ only one of $X_{i', k_1^j}, X_{i', k_2^j}$ is nonzero, and if it is --- it is at most $t_0$ (again by definition of $W_i$). Hence $|b_{i'}| \leq t_0$ as expected. Moreover, the size of support of $b$ is at most $\frac{\sqrt{n}}{2}$, and by assumption that $\theta = \Oh(\frac{1}{\sqrt{n}})$, this value is at most $\frac{1}{8 \theta}$

    Finally, by union bound over all $i \in [n]$, the statement of the lemma holds --- for fixed $i$ with probability at most $\frac{\delta}{2n}$ set $W_i$ fails to be large enough, and conditioned on this set being large, with probability at most $\frac{\delta}{2n}$ it fails to contain two elements of interests.

\end{proof}

\begin{remark}\RemarkName{column-pair}
\textup{
If one uses \erspuddc and not \erspuddctwo, then condition {\bf(P3)} actually {\em requires} $p\gg n^2$ to hold with non-negligible probability. To see this, we first describe the difference between \erspuddctwo and the \erspuddc algorithm of \cite{SpielmanWW12}. In \erspuddc, rather than try all $T = \binom{p}{2}$ pairings of columns as in \erspuddctwo, it only tries $p/2$ column pairs formed by randomly pairing the $p$ columns with each other.
}

\textup{
Now, consider the case of sparsity $s = 1$, i.e.\ $\theta = 1/n$, with $p\gtrsim n\log n$. Then for $i\in [n]$ if we let $q_i$ be the number of columns of $X$ with support containing $i$, then by the Chernoff bound whp for all $i$ we have $q_i = \Theta(p/n)$. Fix an $i$ and consider the $q_i$ columns $j_1,\ldots,j_{q_i}$ containing $i$ in their support. Note that the expected number of these $q_i$ columns that are randomly paired with another one of the same $q_i$ columns by \erspuddc is $q_i (q_i-1)/(p-1) = \Theta(p/n^2)$, which is $o(1)$ for $p = o(n^2)$ so that {\bf (P3)} is likely to fail. In fact, essentially this same argument shows that unless $p = \Omega(n^2\log n)$, it is likely that there will be {\em some} $i\in[n]$ such that none of the $q_i$ columns containing $i$ in its support will be paired with each other. In \Section{bernoulli-rademacher-lower-bound} we show that not only does {\bf (P3)} fail whp for $p = o(n^2)$, but in fact \erspud itself fails for $p \ll n^2$.
}
\end{remark}
\section{Concentration and chaining background}

We now provide some preliminary definitions and results we will need to prove our chaining theorem, \Theorem{main}. As per \Corollary{correctness} and the proof of \Theorem{prob-analysis}, \Theorem{main} fits in to show that \erspud achieves correct recovery with probability $1-\delta$ for $p\gtrsim n\log(n/\delta)$ and $1/n\lesssim \theta \lesssim 1/\sqrt{n}$.

\subsection{Tail bounds}

For a random variable $Z$, we make the standard definition $\psi_Z(\lambda) = \ln \E e^{\lambda Z}$ (e.g.\ \cite[Section 2.2]{Boucheron13}). The following lemma is then immediate.

\begin{lemma}\LemmaName{psi-independent}
If $Z, Z'$ are independent, then $\psi_{Z + Z'}(\lambda) = \psi_Z(\lambda) + \psi_{Z'}(\lambda)$.
\end{lemma}

The following definition and facts concerning subgamma random variables are standard \cite[Section 2.4]{Boucheron13}. 

\begin{definition}\DefinitionName{subgamma}
A random variable $Z$ is said to be {\em $(\sigma, B)$-subgamma} if $\E Z = 0$ and $\psi_Z(\lambda) \le \lambda^2 \sigma^2/(2 (1 - B\lambda))$ for all $|\lambda| < 1/|B|$.
\end{definition}

\begin{lemma}[Basic properties of subgamma random variables]\LemmaName{subgamma}
It holds that
	\begin{enumerate}
		\item If $Z$ is $(\sigma, B)$-subgamma and $\alpha > 0$, then $\alpha Z$ is $(\alpha \sigma, \alpha B)$-subgamma.
		\item If $Z_1,\ldots,Z_n$ are independent and each $Z_i$ is $(\sigma_i, B_i)$-subgamma, then $\sum_{i=1}^n Z_i$ is $(\sqrt{\sum_i \sigma_i^2}, \min_i B_i)$-subgamma.
		\item If $Z$ is $(\sigma, B)$-subgamma, then
			\begin{equation*}
				\P(|Z| > \lambda) \lesssim \exp\left(-\frac{\lambda^2}{2\sigma^2}\right) + \exp\left(-\frac{\lambda}{2 B}\right)
			\end{equation*}
		\item If $Y$ is a symmetric $(\sigma, B)$-subgamma random variable, and $Z$ is symmetric such that $|Z| \leq |Y|$ with probability $1$, then $Z$ is $(\sigma, B)$-subgamma.
	\end{enumerate}
\end{lemma}

\begin{remark}\textup{
The work \cite{LuhV15}, in addition to providing a bound on $p$ sufficient for \Eqsub{random-process} to hold for some $\delta = o(1)$, also advertised a new ``refined version of Bernstein’s concentration inequality for a sum of independent variables''. In fact though, their concentration inequality is equivalent to the statement that a sum of subgamma random variables is subgamma with new parameters as per above, which is a known fact (the reader is encouraged to read the excellent treatment of sums of independent random variables in \cite[Section 2]{Boucheron13}, from which the above definitions and lemmas are taken).
}
\end{remark}

\begin{definition} \DefinitionName{subgaussian}
	A random variable $Z$ is said to be {\em $\sigma$-subgaussian} if $\E Z = 0$ and $\psi_Z(\lambda) \leq \lambda^2 \sigma^2/2$. It is called simply subgaussian if it is 1-subgaussian.
\end{definition}
\begin{fact}
	If $Z$ is a subgaussian random variable, and $D \in \{0,1\}$ is Bernoulli random variable, independent from $Z$ with $\E D = \theta$, then $ZD$ is $(\sqrt{2 \theta}, 1)$-subgamma. \end{fact}
\begin{proof}
Take $\lambda < 1$. We have
\begin{align*}
	\exp(\psi_{D Z}(\lambda)) = \E \exp(D Z \lambda) = 1 - \theta + \theta \E \exp(Z \lambda) \leq 1 - \theta + \theta \exp(\lambda^2/2)
\end{align*}
We can expand $\exp(\lambda^2/2)$, to get
\begin{align*}
	\exp(\psi_{D Z}(\lambda)) & \leq 1 - \theta + \theta \exp(\lambda^2/2) \\
	& = 1 - \theta + \theta \left(\sum_{k=0}^\infty \frac{\lambda^{2k}}{2^k k!}\right) \\
	& = 1 + \frac{\theta \lambda^2}{2} + \frac{\theta \lambda^2}{4} \left( \sum_{k=2}^{\infty} \frac{\lambda^{2k - 2}}{2^{k-2} k!}\right) \\
	& \leq 1 + \frac{\lambda^2 \theta}{2} + \frac{\lambda^2 \theta}{4} \left(\sum_{k=0}^{\infty} 2^{-k}\right) \\
	& = 1 + \lambda^2 \theta \\
	& \leq \exp(\lambda^2 \theta) \\
	& \leq \exp(\frac{\lambda^2 \theta}{1 - \lambda})
\end{align*}

Taking logarithms on both sides proves the result.
\end{proof}
\begin{lemma}[Moment and tail bounds equivalence, Lemma 4.10 \cite{LedouxT91}]
	\LemmaName{moment-tail-bounds}
  Let $Z$ be a nonnegative random variable.
  \begin{enumerate}
  \item If there exists some constants $m, s, b$, such that for every $\lambda_1, \lambda_2$
    \begin{equation}
      \P(Z > m + \lambda_1 s + \lambda_2 b) < \exp(-\lambda_1^2) + \exp(-\lambda_2) ,
      \label{}
    \end{equation}
    
    then for all $p\ge 1$ it holds that $\|Z\|_p \leq C (m + \sqrt{p} s + p b)$, where $C$ is a universal constant.
  \item If for some constants $s,b$ and for all $p\ge 1$ it holds that $\|Z\|_p \leq m + \sqrt{p} s + p b$, then
    \begin{equation}
      \P(Z > C(m + \lambda_1 s + \lambda_2 b)) < \exp(-\lambda_1^2) + \exp(-\lambda_2)
      \label{}
    \end{equation}
    where $C$ is a universal constant.
  \end{enumerate}
\end{lemma}

\subsection{Chaining}

In this subsection we provide relevant definitions for a technique called \emph{generic chaining}, as well as statements of some of the results in the area. Those tools have been designed to provide answers about the supremum of the fluctuations from mean for a large collection of random variables, when the reasonable bounds for covariances in terms of the geometry of the set of indices are at hand. In particular, those methods reduce questions about such fluctuations to questions about purely geometric quantities of the set of indices, and they proved to be extremely useful in a number of applications. The generic chaining method will be a core of the proof of \Theorem{main}. For a more detailed exposition of this technique, we refer the reader to an excellent book on that topic \cite{Talagrand14}. 

\begin{definition}[Admissible sequence]
  For an arbitrary set $T$, we say that a sequence of its subsets $(T_k)_{k=0}^{\infty}$ is admissible if for every number $k$ it is true that $T_{k} \subset T_{k+1}$ and $|T_k| \leq 2^{2^k}$ for $k\ge 1$ and $|T_0| = 1$.
\end{definition}

\begin{definition}[Gamma functionals]
  For a metric space $(T, d)$ we define
  \begin{equation}
    \gamma_{\alpha}(T, d) := \inf_{(T_k)} \sup_{x \in T} \sum_{k=0}^\infty 2^{k/\alpha} d(x, T_k)
    \label{}
  \end{equation}
  where the infimum is taken over all admissible sequences $T_k$.  In the above formula we define as usual $d(x, T_k) := \inf_{t \in T_k} d(x, t)$. 
\end{definition}
\begin{fact}
  If $d$ and $d'$ are two metrics such that for $d(t_1, t_2) = C d'(t_1, t_2)$ for every pair of points $t_1, t_2$, then $\gamma_{\alpha}(T, d) = C \gamma_{\alpha}(T, d')$
\end{fact}

\begin{theorem}[Generic chaining \cite{Talagrand14}, Theorem 2.2.23]
	\TheoremName{chaining}
  Let $T$ be an arbitrary set of indices, and $d_1, d_2 : T\times T \to \R_{\geq 0}$ two metrics on $T$. Suppose that with any point $t \in T$ we have associated random variable $X_t$, with $\E X_t = 0$. Suppose moreover, that for any two points $u, w \in T$ we have a tail bound:
  \begin{equation}
    \P\left(|X_u - X_v| > \lambda\right) \lesssim \exp\left(-\frac{\lambda^2}{d_1(u, v)^2}\right) + \exp\left(-\frac{\lambda}{d_2(u, v)}\right)
    \label{}
  \end{equation}
  Then
  \begin{equation}
    \E \sup_{u \in T} |X_u| \lesssim \gamma_2(T, d_1) + \gamma_1(T, d_2)
    \label{}
  \end{equation}
\end{theorem}

\begin{theorem}[Dirksen, \cite{Dirksen15}]
	\TheoremName{chaining-tails}
	Let $T$ be an arbitrary set of indices and $d_1, d_2: T \times T \to \R_{\geq 0}$ two metrics on $T$. Suppose that with any point $t\in T$ we have associated random variable $X_t$, such that $\E X_t = 0$. Suppose moreover that for any two points $u, w \in T$, we have a tail bound
  \begin{equation*}
    \P(|X_u - X_v| > \lambda) \lesssim \exp\left(-\frac{\lambda^2}{d_1(u, v)^2}\right) + \exp\left(-\frac{\lambda}{d_2(u, v)}\right)
    \label{}
  \end{equation*}
  Then for there exists an universal constant $C$, such that for any $u > 0$
  \begin{equation*}
	  \P\left(\sup_{u \in T} |X_u| > C(\gamma_2(T, d_1) + \gamma_1(T, d_2) + \sqrt{u} \Delta(T, d_1) + u \Delta(T, d_2))\right) < e^{-u}
    \label{}
  \end{equation*}

  where $\Delta(T, d) := \sup_{u,v \in T} d(u, v)$.

\end{theorem}

\begin{remark}\textup{
The work \cite{LuhV15} observed that the method of their proof is connected to generic chaining, but that after a certain point the methods ``become different in all aspects'' \cite[Section G]{LuhV15}. As we will see soon in the proof of \Theorem{main} in \Section{main-proof}, our analysis in fact simply uses the generic chaining results above, black box, without any ad hoc adjustments. Thus, in addition to improving the bounds in \cite{LuhV15}, our proof also has the benefit of using standard chaining results, perhaps thus also making the proof more accessible.
}
\end{remark}

In some special cases it is known that bounds obtained via generic chaining are optimal up to a constant factor. We will use two of such results. Strictly speaking these results are not crucial in our analysis (one could also proceed by constructing near-optimal admissible sequences for the two different sets $T$ that arise in our proof), but invoking these results shrinks the length of our final proof significantly.

\begin{theorem}[Majorizing measures \cite{Talagrand14}, Theorem 2.4.1]
	\TheoremName{majorizing-measures}
  Let $T \subset \R^n$, and assume that $g=(g_1, \ldots g_n)$ is a vector of i.i.d. standard normal random variables. Then
  \begin{equation}
    \E \sup_{t \in T} \inprod{g, t} \simeq \gamma_2(T, d_2)
    \label{}
  \end{equation}
  Where $d_p$ is the metric induced by the $\ell_p$ norm.
\end{theorem}

\begin{theorem}[\cite{Talagrand14}, Theorem 10.2.8]
	\TheoremName{mm-for-exp}
  Let $T \subset \R^n$, and assume that $x=(x_1, \ldots, x_n)$ is a vector of i.i.d. standard exponential random variables. Then
  \begin{equation}
    \E \sup_{t\in T} \inprod{t, x} \simeq \gamma_2(T, d_2) + \gamma_1(T, d_\infty)
    \label{}
  \end{equation}
\end{theorem}

\section{Proof of the stochastic process bound}\SectionName{main-proof}
In this section we will prove the following theorem, which provides a stronger form of \Equation{stoch-process}.

\begin{theorem}\TheoremName{main}
	Let $\Pi \in \R^{m \times n}$ be a random matrix with i.i.d. random entries $\pi_{ij} = \chi_{ij} g_{ij}$, where $\chi_{ij} \in \{0, 1\}$ is a Bernoulli random variable with $\E \chi_{ij} = \theta$, and $g_{ij}$ symmetric subgaussian random variable. Moreover, assume that $\frac{1}{n} \leq \theta$. When $m = \Omega(\varepsilon^{-2} n \log\frac{n}{\delta})$,
	\begin{equation}
		\Pr_\Pi\left(\sup_{v\in B_1} \left| \| \Pi v \|_1 - \E \|\Pi v\|_1 \right| > \varepsilon \cdot \E \| \Pi v \|_1\right) < \delta
		\EquationName{random-process}
	\end{equation}
\end{theorem}

We now prove the theorem. Define $B_1 := \{ t \in \R^n : \|t \|_1 \leq 1 \}$. For each $v \in B_1$, consider
\begin{equation}
	\tilde{X}_v := \|\Pi v\|_1 - \E \|\Pi v\|_1
	\label{}
\end{equation}

We wish to prove that with high probability over $\Pi$ we have 
\begin{equation}
	\sup_{v\in B_1} |\tilde{X}_v| \leq \varepsilon \mu_{min}
	\label{}
\end{equation}
where $\mu_{min} := m\sqrt{\frac{\theta}{n}}$ is such that for every $v \in B_1$ we have $\mu_{min} \leq \E \|\Pi v\|_1$ --- as it was shown in \cite[Lemma 16]{SpielmanWW12}.

Let $\pi_1, \ldots \pi_m$ be rows of matrix $\Pi$. With each $v \in B_1$ we associate another random variable
\begin{equation}
	X_v := \sum_{i=1}^m \sigma_i |\inprod{\pi_i, v}|
	\label{}
\end{equation}

\begin{lemma}
	\LemmaName{symmetrization}
	For every integer $p$ we have
	\begin{equation}
		\| \sup_{v\in B_1} |\tilde{X}_v| \|_p \lesssim \| \sup_{v\in B_1} |X_u| \|_p
		\label{}
	\end{equation}
\end{lemma}
\begin{proof}
	Without loss of generality consider an even $p$, so that $|X_u|^p = X_u^p$. Let $\tilde{\Pi}$ be a random matrix, independent and identically distributed as $\Pi$. By Jensen inequality we have
	\begin{align*}
		\| \sup_{v \in B_1} |\tilde{X}_v| \|_p
		& = \left\| \sup_{v\in B_1} \|\Pi v\|_1 - \E_{\tilde{\Pi}} \| \tilde{\Pi} v\|_1  \right\|_p \\
		& \leq \left\| \sup_{v \in B_1} \|\Pi v\|_1 - \|\tilde{\Pi} v\|_1 \right\|_p \\
		& = \left\| \sup_{v \in B_1} \sum_{i=1}^m |\inprod{\pi_i, v} | - |\inprod{\tilde{\pi}_i, v}| \right\|_p
	\end{align*}

	Now each summand $|\inprod{\pi_i, v}| - |\inprod{\pi_i, v}|$ is symmetric random variable. We can introduce independent random signs $\sigma_i$, and the distribution of a summands is unaffected
	\begin{align*}
		\| \sup_{v \in B_1} |\tilde{X}_v| \|_p 
		& \lesssim \left\| \sup_{v \in B_1} \sum_{i=1}^m \sigma_i( |\inprod{\pi_i, v}| - |\inprod{\tilde{\pi}_i, v}|) \right\|_p \\
		& \leq \left\| \sup_{v \in B_1} \sum_{i=1}^m \sigma_i |\inprod{\pi_i, v}| \right\|_p + \left\| \sup_{v\in B_1} \sum_{i=1}^p (- \sigma_i) |\inprod{\tilde{\pi}_i, v}| \right\|_p \\
		& = 2 \left\| \sup_{v \in B_1} \sum_{i=1}^m \sigma_i |\inprod{\pi_i, v}| \right\|_p \\
		& \lesssim \left\| \sup_{v \in B_1} |X_v| \right\|_p
	\end{align*}
\end{proof}

We will first analyze tail behavior of the random variable $\sup_{v \in B_1} |X_v|$, and then use \Lemma{symmetrization} together with \Lemma{moment-tail-bounds} to obtain tail bounds for the random variable of original interest ($\sup_{v\in B_1} |\tilde{X}_v|$).

In order to use \Theorem{chaining-tails} to obtain tail bounds for supremum of $X_u$, we need to bound tails of random variables $X_u - X_v$ for $u, v \in B_1$.

\begin{lemma}
	\LemmaName{distance}
	For every pair of points $u, v \in B_1$, we have
	\begin{equation}
		\P(|X_u - X_v| > \lambda) \lesssim \exp\left(-\frac{\lambda^2}{2 m \theta \|u - v\|_2^2}\right) + \exp\left(-\frac{\lambda}{\|u - v\|_\infty}\right)
		\label{}
	\end{equation}
\end{lemma}
\begin{proof}
	We can write
	\begin{equation}
		X_u - X_v = \sum_{i=1}^m \sigma_i ( |\inprod{\pi_i, u}| - |\inprod{\pi_i, v}|)
		\label{}
	\end{equation}
	Define $Q_i := \sigma_i (|\inprod{\pi_i, u}| - |\inprod{\pi_i, v}|)$. We have $X_u - X_v = \sum_{i=1}^m Q_i$, where all $Q_i$ are symmetric and identically distributed. 

	Moreover, we have $|Q_i| = \left| |\inprod{\pi_i, u}| - |\inprod{\pi_i, v}| \right| \leq |\inprod{\pi_i, u-v}|$. Observe that each $\pi_{ij}$ is $(\sqrt{2 \theta}, 1)$-subgamma. Therefore, by basic properties of subgamma random variables (\Lemma{subgamma}) we know that $\inprod{\pi_{i}, u-v}$ is $(\sqrt{2\theta} \|u - v\|_2, \|u - v\|_\infty)$-subgamma.

	Now, as both $Q_i$ and $\inprod{\pi_{i}, u-v}$ are symmetric, and always $|Q_i| \leq |\inprod{\pi_{i}, u-v}|$, we deduce that each $Q_i$ is also $(\sqrt{2 \theta} \|u - v\|_2, \|u - v\|_\infty)$-subgamma. 

	Finally, $X_u - X_v$, as a sum of independent subgamma random variables is $( \sqrt{2 m \theta} \|u - v\|_2^2, \|u - v\|_\infty)$-subgamma. This, together with \Lemma{subgamma} implies tail bound
	\begin{equation}
	\P\left(\left|\sum_{i=1}^m Q_i\right| > \lambda\right) \lesssim\exp\left(\frac{\lambda^2}{2 m \theta \|u - v\|_2^2}\right) + \exp\left(\frac{\lambda}{\|u - v\|_\infty}\right)
		\label{}
	\end{equation}
\end{proof}

With this lemma in hand, we can use \Theorem{chaining-tails}, to deduce the tail bound for supremum of $|X_v|$.
\begin{equation}
	\P\left(\sup_{v \in B_1} |X_u| > M + \sqrt{u} D_1 + u D_2\right) < e^{-u}
	\EquationName{tail-bound}
\end{equation}
Where
\begin{align*}
	M  & := C_1 (\gamma_2(B_1, \sqrt{2 m \theta} d_2) + \gamma_1(B_1, d_\infty))\\
	D_1 & := C_2 \Delta(B_1, \sqrt{2 m \theta} d_2) \\
	D_2 & := C_3 \Delta(B_1, d_\infty)
\end{align*}
with $d_2, d_\infty$ --- metrics on $\mathbb{R}^n$ induced by norms $\ell_2, \ell_\infty$ respectively, and $C_1, C_2, C_3$ are universal constants.

We claim that, we can deduce similar tail bounds for  $\sup_{v \in B_1} |\tilde{X}_u|$. Namely
\begin{equation}
	\P\left(\sup_{v \in B_1} |\tilde{X}_u| > L(M + \sqrt{u} D_1 + u D_2)\right) < e^{-u}
	\EquationName{tail-bound-orig}
\end{equation}
for some universal constant $L$.

Indeed by \Lemma{moment-tail-bounds}, tail bound of form \Equation{tail-bound} implies moment bounds of form $\| \sup_{v \in B_1} |X_v| \|_p \lesssim M + \sqrt{p} D_1 + p D_2$.  By \Lemma{symmetrization}, the same (up to a constant) moment bounds are true for $\sup_{v\in B_1} |\tilde{X}_v|$. Finally, applying the other direction of \Lemma{moment-tail-bounds} we deduce similar tail behavior of random variable $\sup_{v\in B_1}\tilde{X}_v$, as in \Equation{tail-bound-orig}.

If we set $u := \log \frac{1}{\delta}$ in \Equation{tail-bound-orig}, we will get an upper bound for $\sup_{v \in B_1} \tilde{X}_v$ which is satisfied with probability at least $1 - \delta$. We need to understand the values of $M$, $\sqrt{u} D_1$ and $u D_2$, for this setting of $u$, and we will show how to pick $m$ such that sum of those values is smaller than $\varepsilon \mu_{min}$.

Let us focus now on bounding $M$. We have $\gamma_2(B_1, \sqrt{2 m\theta} d_2) = \sqrt{2 m \theta} \gamma_2(B_1, d_2)$. We need an upper bound for $\gamma_2(B_1, d_2)$ and $\gamma_1(B_1, d_\infty)$.

\begin{fact}
	$\gamma_2(B_1, d_2) \lesssim \sqrt{\log n}$
\end{fact}
\begin{proof}
	By \Theorem{majorizing-measures}, we have
	\begin{equation}
		\gamma_2(B_1, d_2) \lesssim \E_g \sup_{t \in B_1} \inprod{t, g}
		\label{}
	\end{equation}
	where $g$ is a Gaussian vector. By the duality of $\ell_1$ and $\ell_\infty$ norms, for any vector $w\in\R^n$ we have $\sup_{t\in B_1} \inprod{t, v} = \|v\|_\infty$, so in particular
	$\E \sup_{t\in B_1} \inprod{t, g} = \E \|g\|_\infty \simeq \sqrt{\log n}$. 
\end{proof}

\begin{fact}
$\gamma_1(B_1, d_\infty) \lesssim \log n$
\end{fact}
\begin{proof}
	By \Theorem{mm-for-exp}, we have
	\begin{equation}
		\gamma_1(B_1, d_\infty) \lesssim \E_x \sup_{t\in B_1} \inprod{t, x}
		\label{}
	\end{equation}
	where $x$ is a vector of independent standard exponentially distributed random variables. Again $\sup_{t\in B_1} \inprod{t, x} = \|x\|_\infty$. It is standard fact that $\E \|x\|_\infty \simeq \log n$. 
\end{proof}
Those two facts together with previous discussion yield an upper bound $M \lesssim \sqrt{m \theta \log n} + \log n$.

Moreover, as $d_2(u, v) \leq d_1(u,v)$ for any $u, v \in \R^n$, where $d_1$ is metric induced by $\ell_1$ norm, we can easily upper bound diameter of $B_1$ in $d_2$ by diameter of $B_1$ and $d_1$ and therefore obtain an upper bound for $D_1$
\begin{equation*}
	D_1 = C_1 \Delta(B_1, \sqrt{e m \theta} d_2) = C_1 \sqrt{e m \theta} \Delta(B_1, d_2) \leq C_1 2 \sqrt{e m \theta}
\end{equation*}
and similarly $\Delta(B_1, d_\infty) = 2$. Altogether, we have following inequalities
\begin{align*}
    M & \lesssim \sqrt{m \theta \log n} + \log n \\
    D_1 & \lesssim \sqrt{m \theta} \\
    D_2 & \lesssim 1
    \label{}
\end{align*}
Plugging this back to \Equation{tail-bound-orig}, we have 
\begin{equation}
	\P\left(sup_{v\in B_1} |\tilde{X}_v| < L_2 \left(\sqrt{m \theta (\log n + \log \frac{1}{\delta})} + \log n + \log\frac{1}{\delta}\right)\right) < \delta
	\EquationName{prev-tail-bound}
\end{equation}
where again $L_2$ is some constant. 

The following inequalities are equivalent:
\begin{align*}
	L_2 \sqrt{m \theta \log \frac{n}{\delta}} & \leq \frac{1}{2} \varepsilon \mu_{min} \\
	L_2 \sqrt{m \theta \log \frac{n}{\delta}} & \leq \frac{1}{2} \varepsilon \sqrt{\frac{\theta}{n}} m \\
	\frac{4L_2^2}{\varepsilon^2} n \log \frac{n}{\delta} & \leq m
\end{align*}

Similarly, assumption $\theta \geq \frac{1}{n}$, implies that if $m > \frac{2 L_2}{\varepsilon} n \log \frac{n}{\delta}$, then also $L_2 \log \frac{n}{\delta} \leq \frac{1}{2} \varepsilon \mu_{min}$, so once $m$ is larger than both those values, \Equation{prev-tail-bound} implies
\begin{equation}
	\P\left( \sup_{v \in B_1} |\tilde{X}_v| > \varepsilon \mu_{min}\right) < \delta
	\label{}
\end{equation}
as desired.

\section*{Acknowledgments}
We thank John Wright for pointing out to us the recent independent work \cite{Adamczak16}.

\bibliography{biblio}{}
\bibliographystyle{alpha}

 \newpage

 \appendix

 \section*{Appendix}

\section{Lower bounds for sample complexity of \erspuddc algorithm in the Bernoulli-Rademacher model}\SectionName{bernoulli-rademacher-lower-bound}
In this section we prove that the modification introduced in this paper to \erspuddc algorithm is necessary in order to guarantee the correctness for arbitrary Bernoulli-subgaussian $X$ with strictly subquadratic number of samples $p$. More concretely, we prove that if $X$ follows the Bernoulli-Rademacher model, and $p = \Oh(n^{2 - \varepsilon}\log n)$, then the \erspud algorithm actually fails to recover $A$ and $X$ with probability at least $1 - \tilde{\Oh}(\frac{1}{n^{\varepsilon}})$. The proof of this theorem relies on few technical lemmas, they are presented later in this section.
\begin{theorem}\TheoremName{main-lb}
    For every constant $C$ and $\varepsilon \leq 1$, there exist $C'$, such that for sufficiently large $n$ if $2 n \leq p \leq C n^{2 - \varepsilon} \log n$ and $X \in \R^{n\times p}$ follows the Bernoulli-Rademacher model with sparsity parameter $\theta := C' \frac{\log n}{n}$, then the \erspuddc algorithm fails to recover $X$ with probability at least $1 - \Oh(\frac{\log^5 n}{n^\varepsilon})$.
\end{theorem}
\begin{proof}
    We shall first prove that once following events happens simultaneously, the \erspuddc algorithm fails to recover $X$. Later on we will prove that each of those events fails with probability at most $\Oh(\frac{\log^5 n}{n^\varepsilon})$ --- that will be enough to conclude the statement of the theorem. 

    In what follows, let $X_i$ be the $i$-th column of $X$, $j_* \in [n]$ be the index of the row of $X$ with largest number of non-zero entries (for concreteness, the smallest such index), and $K$ be some universal constant the same for fourth and fifth event, and will be specified later. Consider the following events
     \begin{enumerate}
         \item Matrix $X$ is of full rank.
         \item For every $i \in [p-1]$ it holds that $|\supp(X_i) \cap \supp(X_{i+1})| < 2$.
         \item For every $i \in [p-1]$ it holds that $j_* \not\in \left(\supp(X_i) \cap \supp(X_{i+1})\right)$.
         \item Every column of $X$ has at least $K \log n$ nonzero entries.
         \item The number of rows of $X$ with largest support size is smaller than $K \log n$. 
     \end{enumerate}

    Let us call those events $\mathcal{E}_1, \ldots, \mathcal{E}_5$ respectively. We claim that under $\mathcal{E}_1, \ldots, \mathcal{E}_5$, the $j_*$-th row of $X$ would not be recovered by the \erspuddc algorithm. Indeed, assume for the proof by contradiction, that solving the optimization problem $\min_{w} \|w^T Y\|_1$ subject to $r_*^T w = 1$, yields a solution such that $w^T Y$ is proportional to the $j_*$-th row of $X$, for some $r_*$ which is sum of two consecutive columns of $Y$. By condition $\mathcal{E}_1$, it means that the solution to the equivalent problem of $\min_{z} \|z^T X\|_1$ subject to $b_*^T z = 1$ is $z = \pm e_{j_*}$, where $b_*$ is a sum of corresponding columns of $X$.

    Observe that, because the matrix $X$ has entries in $\{-1, 0, 1\}$, the $\ell_1$ norm and sparsity of each row is equal, that is for every $k$ we have $\| e_k^T X \|_1 = \|e_k^T X\|_0$.

     Now, by condition $\mathcal{E}_2$, at most one coordinate of $b_*$ is of absolute value $2$, and all other are either $\pm 1$ or $0$. Moreover, if the condition $\mathcal{E}_3$ holds, the entry with absolute value $2$ is not $j_*$-th. If for some $k$, we had $(b_*)_k = 2$, then taking $z := \frac{1}{2} e_k$, would yield a feasible solution to the optimization problem mentioned above, and one with smaller value of objective function --- $\frac{1}{2}\|e_k^T X\|_0$ as opposed to $\|e_{j_*}^T X\|_0$; similarly for $(b_*)_k = - 2$. Therefore, all non-zero entries of $b_*$ must have absolute value $1$.
     
 Moreover $b_*$ has support of size at least $K \log n$ (as $\mathcal{E}_4$ holds and $b_*$ is a sum of two columns of $X$ with almost disjoint support by $\mathcal{E}_2$), hence $|\supp b_*|$ is strictly larger than the number of rows of $X$ with largest support --- this number of rows is less than $K \log n$ by $\mathcal{E}_5$. In particular, there is some $k \in \supp(b_*)$, such that the $k$-th row of $X$ has strictly smaller support size than $j_*$-th row. Again, if this is the case $z = e_{j_*}$ is not a solution to the optimization problem $\min_{z} \|z^T X\|_1$ subject to $b_*^T z = 1$ --- as $z := e_k$ (or $-e_k$) is feasible and with strictly smaller objective value.
     From this contradiction we conclude, that once all the events in preceding list hold simultaneously, the \erspuddc algorithm fails in recovering $X$, and therefore fails to recover the hidden decomposition.

     It is now enough to show that each of those events $\mathcal{E}_1 \ldots \mathcal{E}_5$ fails with probability at most $\Oh(\frac{\log^5 n}{n^\varepsilon})$.

     Event $\mathcal{E}_1$ fails with probability at most $n(1 - \frac{C'\log n}{n})^n \lesssim \frac{1}{n^{C' - 1}}$ by \Lemma{X-full-rank} and assumption that $p > 2n$. For $C' \geq 1$ and large enough $n$ this quantity is smaller than $\frac{\log^5 n}{n^\varepsilon}$.

     For event $\mathcal{E}_2$, it holds with probability $1 - \Oh(\frac{\log^5}{n^\varepsilon})$ simply by union bound --- for every fixed $i$, we have $\P(\supp(X_i) \cap \supp(X_{i+1}) \geq 2) \leq \binom{n}{2} \theta^4 \lesssim \frac{\log^4 n}{n^2}$ --- and we need a union bound over $p \leq C n^{2 - \varepsilon} \log n$ such events.

     To bound the probability of event $\mathcal{E}_3$, let random set $S \subset [p]$ be the support of $j_*$-th row of $X$. In what follows we will condition implicitly on $S \not= \emptyset$ as $S$ is empty with exponentially small probability.
     
     Expected support size of any single row is $p\theta > 2 \log n$, therefore by Chernoff and union bound, we deduce that except with probability smaller than $\frac{1}{n}$ all rows of $X$ has support size smaller than $C_2 p\theta$ for some universal constant $C_2$. Conditioned on $|S| < C_2 p \theta$, the distribution of $S$ is invariant under permutations of $[n]$, in a sense that for fixed set $S_0\subset [n]$ probability $\P(S = S_0 | |S| < C_2 p \theta)$ depends only on the size of $S_0$. In such a case, and because of additional conditioning on $S$ being nonempty, by \Lemma{S-negative-correlation} for every $i \not= j$ we have 
     \begin{equation*}
         \P(i \in S | j \in S \land |S| < C_2 p \theta) \leq \P(i \in S | |S| < C_2 p \theta)
     \end{equation*}

     In particular, for fixed $i \in [p-1]$, we have 
     \begin{align*}
         \P\left(i \in S \, \land\, (i+1) \in S \,|\, |S| < C_2 p \theta \right) & \leq \P\left(i \in S | |S| < C_2 p \theta\right) \P\left(i+1 \in S | |S| < C_2 p \theta \right)  \\
         & \leq \frac{C_2^2 p^2 \theta^2}{p^2} \\
         & \lesssim \frac{\log^2 n}{n^2}
     \end{align*}
     Hence, by union bound over all $i \in [p-1]$, it follows that 
     \begin{equation*}
         \P(\lnot \mathcal{E}_3) \leq \P(|S| \geq C_2 p \theta) + \P(\lnot \mathcal{E}_3 | |S| < C_2 p \theta) \lesssim \frac{\log^3 n}{n^{\varepsilon}}
     \end{equation*}

     For $\mathcal{E}_4$, we know that the expected number of non-zero entries in a column of $X$ is $\theta n = C' \log n$ --- by the Chernoff bound, probability that any such column has sparsity smaller than $K \log n$ is much smaller than $\frac{1}{n^4}$ if we set $C'$ large enough depending on $K$. Therefore by union bound, they all have sparsity at least $K \log n$ simultaneously with probability at least $1 - \frac{1}{n^2}$. 

     In order to bound probability of failure for the event $\mathcal{E}_5$, let $s_i\in \N$ for $i \in [n]$ be the size of the support of the $i$-th row of $X$. Clearly $s_i$ are Binomial random variable with parameters $(p, \theta)$.  Take $\gamma := \frac{K_1 \log n}{n}$ (with some constant $K_1$ that will be specified later) and let $T_0 \in [p]$ be the largest number such that $\P(s_i \geq T_0) \geq \gamma$. We want to apply \Lemma{no-large-gaps} for all random variables $s_i$. Observe that in this setting $\E s_i = p\theta \ll \frac{p}{8}$, and on the other hand $\E s_i \geq 2 \log n$.
     
     Moreover, observe that $T_0 \leq 4 \E s_i$ --- it is enough to show that $\P(s_i \geq 4 \E s_i) \leq \gamma$, and this fact follows from Chernoff bound if $K_1$ is large enough constant. Therefore, we can apply \Lemma{no-large-gaps} to conclude that 
     \begin{equation*}
         \P(S_i \geq T_0) \leq K_2 \gamma \max(1, \frac{T_0}{\E S_i}) \leq K_3 \gamma
         \label{}
     \end{equation*}
     where $K_2$ and $K_3$ are some universal constants.

     Now we want to show that with probability at least $1 - \frac{1}{n}$, number of $s_i$ that are not smaller than $T_0$, is between $1$  and  $4 K_3 \gamma n = 4 K_3 K_1 \log n =: K_4 \log n$. If we consider indicator random variables $M_i \in \{0, 1\}$, such that $M_i = 1$ if and only if $S_i \geq T_0$, by previous discussion we know that $\gamma \leq \P(M_i = 0) \leq K_3 \gamma$, and all $M_i$ are independent. We can now apply the Chernoff bound to bound the probability of $\P(\sum M_i < 1)$ and $\P(\sum M_i \geq 4 K_3 \gamma n)$ --- again, if $K_1$ is large enough, each of those is much smaller than $\frac{1}{2n}$ --- we can now fix constant $K_1$ large enough so that all three Chernoff bounds yield desired inequalities.

     Finally, if the number of rows with support larger than $T_0$ is between one and $K_4 \log n$, then clearly at most $K_4 \log n$ has the largest support --- and therefore the event $\mathcal{E}_5$ holds with $K := K_4$.
\end{proof}

We now prove certain technical lemmas that were used in the proof of \Theorem{main-lb}.

\begin{lemma}
    \LemmaName{no-large-gaps}
    For $\theta \in (0,1)$ and $p\in \N$, let $Q$ be a Binomial random variable with parameters $(p, \theta)$ and let $\gamma \in (0,1)$ be some fixed threshold. Moreover, let $T_0$ be the largest natural number such that
    \begin{equation*}
        \P(Q \geq T_0) \geq \gamma
    \end{equation*}
    and assume that $T_0 \leq \frac{p}{2}$. Then
    \begin{equation*}
        \P(Q \geq T_0) \leq K \gamma \max(1, \frac{T_0}{\E Q})
    \end{equation*}
    for some universal constant $K$.
\end{lemma}
\begin{proof}
    We start with bounding the ratio
    \begin{align*}
        \frac{\P(Q = T_0)}{\P(Q = T_0 + 1)} & = \frac{ \binom{p}{T_0} \theta^{T_0} (1-\theta)^{p - T_0}}{\binom{p}{T_0 + 1} \theta^{T_0 + 1} (1-\theta)^{p - T_0 - 1}} \\
        & = \frac{T_0 + 1}{p - T_0} \cdot \frac{1 - \theta}{\theta} \\
        & \leq K_1 \frac{T_0}{p \theta} \\
        & = K_1 \frac{T_0}{\E Q}
    \end{align*}

    We can rephrase it as $\P(Q = T_0) \leq K_1 \frac{T_0}{\E Q} \P(Q = T_0 + 1)$, and clearly $\P(Q = T_0 + 1) \leq \P(Q \geq T_0 + 1)$. Therefore

    \begin{equation}
        \P(Q = T_0) \leq K_1 \frac{T_0}{\E Q} \P(Q \geq T_0 + 1)
        \label{}
    \end{equation}

    We can now directly bound the desired probability as follows
    \begin{align*}
        \P(Q \geq T_0) & = \P(Q = T_0) + \P(Q \geq T_0 + 1) \\
        & \leq \left(K_1 \frac{T_0}{\E Q} + 1 \right) \P( Q \geq T_0 + 1) \\
        & \leq \left(K_1 \frac{T_0}{\E Q} + 1 \right) \gamma 
        \label{}
    \end{align*}
    Where the last inequality follows from the assumption that $T_0$ were largest such that $\P(Q \geq T_0) \geq \gamma$.
\end{proof}

\begin{lemma}
    \LemmaName{X-full-rank}
    Let $X\in \R^{n\times p}$ follow the Bernoulli-Rademacher model with sparsity parameter $\theta$ and $p > n$. Then matrix $X$ is of rank $n$ with probability at least $1 - n\left(1 - \theta\right)^{p-n}$.
\end{lemma}
\begin{proof}
    Let $W_i \subset \mathbb{R}^p$ be a subspace of $\mathbb{R}^p$ spanned by first $i-1$ rows of $X$. We wish to prove that $i$-th row of $X$ lies in $W_i$ with probability at most $\left(1 - \theta\right)^{p - n}$ --- if we do this, the claim will follow by the union bound. Fix some $i$, and let $v$ be the $i$-th row of $X$; moreover, let $W^{\bot}$ be the orthogonal complement of $W_i$. Clearly $\dim W^{\bot} \geq p - n$.  We will show that for any fixed $W^{\bot}$
    \begin{equation*}
        \P(v \bot W^{\bot}) \leq (1-\theta)^{\dim W^{\bot}}
    \end{equation*}

    Indeed, let $q := \dim W^{\bot}$, and consider sequence of indices $i_1, \ldots i_q$ together with a basis $u_1, \ldots u_q$ of $W^{\bot}$ such that for every $r$ we have $(u_r)_{i_r} \not= 0$, and for every pair $s < r$ we have $(u_s)_{i_r} = 0$ --- such a basis and sequence of indices exists by Gaussian elimination. Now, by the chain rule, we have
    \begin{align}
        \P\left(v \bot W^{\bot}\right) & = \prod_{r=1}^q \P\left(\langle v, u_r \rangle = 0 | \forall_{s < r} \langle v, u_s \rangle = 0\right) \EquationName{v-on-W}
    \end{align}
    Let us fix some $r$ now. We want to show that $\P\left(\langle v, u_r \rangle = 0 | \forall_{s < r} \langle v, u_r \rangle = 0\right) < (1-\theta)$. Observe that the event $\forall_{s < r} \langle v, u_s \rangle = 0$ is independent of $v_{i_r}$; moreover, if we fix all values of $v_j$ for $j\not= i_r$, probability of $\langle v, u_r \rangle = 0$ is at most $(1-\theta)$ --- there is at most one value of $v_{i_r}$ that would make this inner product equal zero, and $v_{i_r}$ assumes every value with probability at most $(1-\theta)$. Therefore
    \begin{align*}
        \P\left(\langle v, u_r \rangle = 0 | \forall_{s < r} \langle v, u_r \rangle = 0\right) & = \E\left(\P\left(\langle v, u_r \rangle = 0 | v_1, \ldots \hat{v_{i_r}} \ldots v_p\right) | \forall_{s < r} \langle v, u_r\rangle = 0\right) \\
        & \leq \E \left( 1-\theta | \forall_{s < r} \langle v, u_r\rangle = 0\right) \\
        & = 1 - \theta
    \end{align*}

    Where $v_1, \ldots \hat{v_{i_r}} \ldots v_d$ denotes omitting the $i_r$-th index in this sequence.

    We can plug this back to \Equation{v-on-W} to conclude that $\P\left(v \bot W^{\bot}\right) \leq (1-\theta)^{\dim W^{\bot}}$ and the statement of the lemma follows.
\end{proof}

\begin{lemma}
    Let $S \subset [n]$ be a random set, with permutationally invariant distribution, i.e. such that for fixed $S_0$, $\P(S = S_0)$ depends only on the size of $S_0$. Assume moreover, that $S$ is nonempty almost surely. Then for any $i \neq j \in [n]$, we have $\P(i \in S | j \in S) \leq \P(i \in S)$. \LemmaName{S-negative-correlation}
\end{lemma}
\begin{proof}
    For $k \in \{0, \ldots n\}$, let $p_k := \P(|S| = k)$. Observe that for fixed $i\in S$
    \begin{equation}
        \P(i \in S) = \sum_{k=1}^n p_k \P(i \in S | |S| = k) = \sum_{k=1}^n p_k \frac{k}{n}
        \label{ }
    \end{equation}

    On the other hand
    \begin{align*}
        \P(i \in S | j \in S) & = \frac{1}{1 - p_0} \left( \sum_{k=1}^n p_k \P(i\in S | j\in S, |S| = k) \right) \\
                              & = \frac{1}{1 - p_0} \left( \sum_{k=1}^n p_k \frac{k-1}{n-1} \right) \\
                              & = \sum_{k=1}^n p_k \frac{k-1}{n-1}
        \label{}
    \end{align*}
    where the last equality follows from the assumption $\P(S = \emptyset) = 0$. 

   Then the statement of the lemma follows by explicitly comparing two expressions for $\P(i \in S)$ and $\P(i \in S | j \in S)$, and using inequality $\frac{k-1}{n-1} \leq \frac{k}{n}$. 
\end{proof}

\end{document}